\documentclass[sigconf,nonacm]{acmart}
\usepackage{amsmath}
\usepackage{booktabs}
\usepackage{multirow}
\usepackage[ruled,vlined,linesnumbered]{algorithm2e}
\usepackage{subcaption}
\usepackage{enumitem}

\theoremstyle{acmplain}
\newtheorem{assumption}{Assumption}
\newtheorem{proposition}{Proposition}

\SetKwInOut{Input}{Input}
\SetKwInOut{Output}{Output}
\SetKwComment{Comment}{$\triangleright$\ }{}

\AtBeginDocument{%
  }

\setcopyright{acmlicensed}
\copyrightyear{2026}
\acmYear{2026}

\begin{document}

\title{Multi-channel Uplift Policy Learning}

\author{Changjian Liu}
\authornote{Work done during an internship at Taobao, Alibaba Group.}
\orcid{0009-0008-6235-6981}
\affiliation{
  \institution{Peking University}
  \city{Beijing}
  \country{China}
}
\email{cjliu25@stu.pku.edu.cn}

\author{Tianyu Wang}
\orcid{0009-0006-0292-2078}
\affiliation{
  \institution{Alibaba Group}
  \city{Beijing}
  \country{China}
  }
\email{yves.wty@@alibaba-inc.com}

\author{Xiaoxuan Deng}
\affiliation{%
  \institution{Alibaba Group}
  \city{Beijing}
  \country{China}
}
\orcid{0009-0004-2650-1492}
\email{dengxiaoxuan.dxx@alibaba-inc.com}

\author{Wentao Zhu}
\authornotemark[1]
\affiliation{%
  \institution{Beihang University}
  \city{Beijing}
  \country{China}
  }
\orcid{0009-0009-8359-285X}
\email{zy2406229@buaa.edu.cn}

\author{Yuwei Xu}
\authornotemark[1]
\affiliation{
  \institution{CUHK-shenzhen}
  \city{Shenzhen}
  \country{China}
}
\email{yuweixu@link.cuhk.edu.cn}
\orcid{0000-0002-7240-6435}

\author{Junqi Jin}
\correspondingauthor
\affiliation{%
  \institution{Alibaba Group}
  \city{Beijing}
  \country{China}
  }
\orcid{0000-0003-2424-2744}
\email{junqi.jjq@alibaba-inc.com}

\author{Yong Gao}
\correspondingauthor
\affiliation{
  \institution{Peking University}
  \city{Beijing}
  \country{China}
  }
\email{gaoyong@pku.edu.cn}
\orcid{0000-0003-1562-6228}

\author{Chuan Yu}
\affiliation{
  \institution{Alibaba Group}
  \city{Beijing}
  \country{China}
  }
\email{yuchuan.yc@alibaba-inc.com}
\orcid{0000-0001-8094-1545}

\author{Jian Xu}
\affiliation{
  \institution{Alibaba Group}
  \city{Beijing}
  \country{China}
  }
\email{xiyu.xj@taobao.com}
\orcid{0000-0003-3111-1005}

\author{Bo Zheng}
\affiliation{
  \institution{Alibaba Group}
  \city{Beijing}
  \country{China}
  }
\email{bozheng@alibaba-inc.com}
\orcid{0000-0002-4037-6315}

\renewcommand{\shortauthors}{Liu et al.}

\begin{abstract}
E-commerce platforms must allocate fixed marketing budgets across multiple channels to maximize business utility. However, standard predict-then-optimize (PTO) paradigms fail in this compositional space due to observational confounding and severe extrapolation. We formulate this challenge as a simplex-constrained uplift decision problem and propose \textsc{ReAlloc}, a fast-slow causal framework. Specifically, an agile Orthogonal Teacher extracts unbiased local gradients from short-term logs, while an Explanation-Guided Student distills them into a structured marginal field over long-term horizons. This design enables support-aware, conservative decisions that capture cross-channel substitutions. Extensive simulations and large-scale online A/B tests on Taobao platform demonstrate that \textsc{ReAlloc} achieves simultaneous lifts in both pay order and income.
\end{abstract}

\keywords{uplift, resource allocation, decision making, e-commerce marketing}

\ccsdesc[500]{Applied computing~Electronic commerce}
\ccsdesc[500]{Information systems~Decision support systems}
\ccsdesc[300]{Information systems~Computational advertising}
\ccsdesc[500]{Computing methodologies~Causal reasoning and diagnostics}

\maketitle

\section{Introduction}

In modern e-commerce systems, platforms increasingly act as centralized decision makers. Specifically, they must allocate limited marketing resources across heterogeneous items and diverse intervention channels \cite{albert2022ecommerce, deng2023multichannel}.
Moving beyond mere outcome prediction, the core objective is to learn intervention policies that causally drive business utility, such as sales, income, and Gross Merchandise Volume (GMV). This naturally frames an \textit{uplift policy learning} problem: given a pre-decision state, the platform must allocate resources to maximize the incremental outcome \cite{gutierrez2017causal,athey2021policy,olaya2020survey}.

\begin{figure*}
    \centering
    \includegraphics[width=0.75\linewidth]{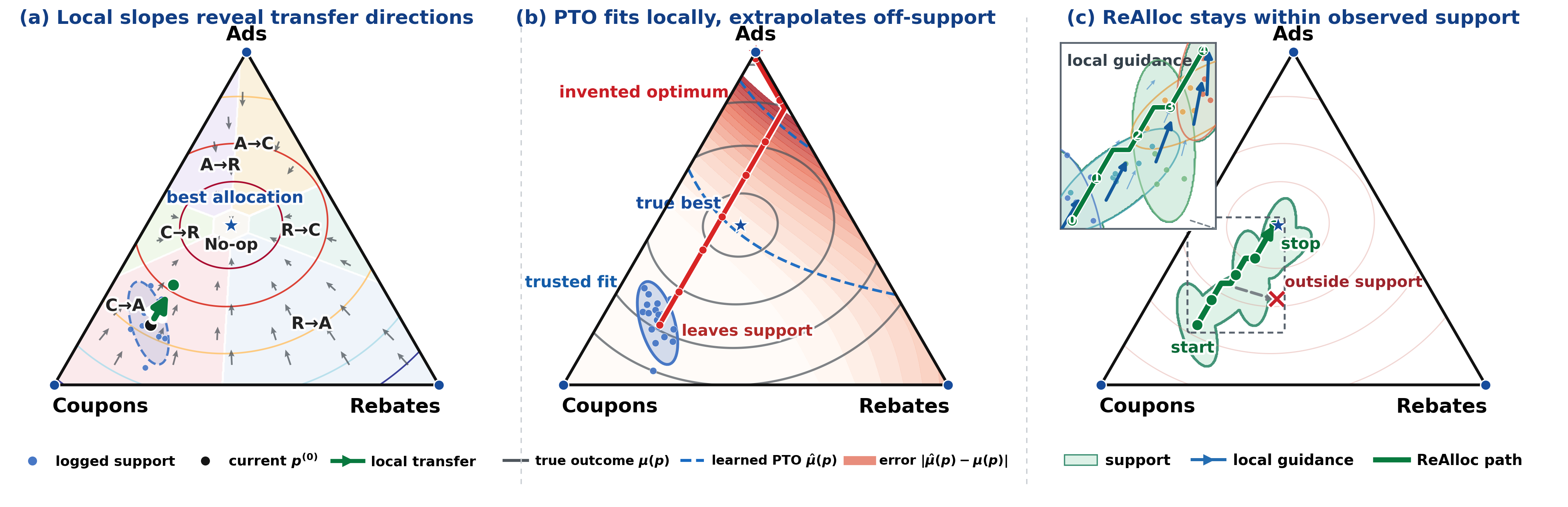}
    \caption{Conceptual illustration of \textsc{ReAlloc}. (a) Local response slopes identify promising reallocation directions. (b) Global PTO extrapolation produces unreliable decisions. (c) ReAlloc composes local slopes within support, leading to better decisions.}
    \label{fig:teasser}
\end{figure*}

Most existing uplift and contextual decision methods focus on binary, discrete, or single-channel treatments \cite{zhao2017uplift, zhao2019uplift}. However, real-world systems are inherently multi-channel and resource constrained. This transforms the intervention from an independent scalar action into a compositional allocation vector constrained on a budget simplex. Such a constraint fundamentally alters the decision landscape from the absolute individual treatment effects (ITE) to the \textit{relative marginal value} of reallocating resources across channels.
A concrete example arises in marketing hosting systems \cite{wang2026marketing}, where a fixed total budget is split between traffic-driving channels (e.g., advertising) and conversion enhancing benefits (e.g., coupons and rebates). Although potentially complementary, these channels frequently substitute or cannibalize each other. For instance, shifting funds to advertising increases exposure but dilutes conversion incentives. Consequently, the final outcome is governed by a complex joint response surface.

A common industrial paradigm is predict-then-optimize (PTO): fit a response model, then search to maximize the predicted outcome \cite{elmachtoub2022smart,wilder2019decision}. However, PTO is fundamentally insufficient for multi-channel uplift learning for three reasons. \textbf{First}, channel-wise models ignore cross-channel interactions; their independent curves cannot recover the optimum of the joint response surface. \textbf{Second}, even a joint black-box model is unsafe under global optimization; the optimizer may exploit extrapolation errors in unsupported regions, yielding aggressive policies. \textbf{Third}, and most critically, prediction accuracy does not imply decision quality. Standard losses (e.g., MSE) optimize average prediction, whereas deployment relies on counterfactual ranking \cite{devriendt2022learning}. Minor prediction errors in high-uplift regions are severely amplified by the optimizer, ultimately degrading the result.
This optimization vulnerability is further exacerbated by observational confounding. Historical allocations are driven by legacy business policies, systematically entangling treatment assignments with item states. Consequently, models achieve strong factual prediction but yield biased causal gradients, rendering downstream optimization ineffective.

To address these challenges, we propose \textsc{ReAlloc}, a framework targeting local reallocation for multi-channel uplift decision problem. It learns the causal marginal gradient of the budget simplex. First, we use an orthogonalized teacher to remove confounding effects and recover unbiased local gradients. These gradients are then distilled into a student marginal field that guides conservative decisions within the observed support. 
\textsc{ReAlloc} turns unsafe PTO into support-aware causal reallocation, enabling precise and stable optimization under multi-channel budget constraints.
Extensive offline experiments and large-scale online A/B tests on the Taobao platform demonstrate its superiority. It significantly outperforms uplift and PTO baselines in counterfactual ranking and decision quality. More importantly, online results reveals substantial improvements in both pay order 3.53 \% and income 3.26 pt.
Our contributions are: 
\begin{enumerate}[leftmargin=*,topsep=3pt,itemsep=3pt,parsep=0pt,partopsep=0pt]
    \item \textbf{Compositional Uplift.} We formulate multi-channel budget allocation as compositional uplift on the simplex, targeting relative marginal effects under zero-sum constraints.
    
    \item \textbf{Support-Aware Causal Reallocation.}  We propose \textsc{ReAlloc}, which decouples causal response learning from support-aware decisions through a distilled conservative marginal field.
    
    \item \textbf{Industrial-Scale Deployment.} Deployed in Taobao's system, \textsc{ReAlloc} simultaneously improves pay order and income.
\end{enumerate}
\section{Related Work}
\label{sec:related_work}

\subsection{Uplift Modeling}
Uplift modeling and heterogeneous treatment effect (HTE) estimation aim to quantify the incremental impact of interventions \cite{radcliffe2007using, gutierrez2017causal, kunzel2019metalearners, wager2018estimation}. While classical and modern neural estimators have achieved significant success, they primarily focus on binary or discrete treatments, targeting absolute individual treatment effects (ITE) \cite{shalit2017estimating, shi2019adapting, nie2021quasi, zhong2022descn}. Recent extensions consider multiple treatments \cite{olaya2020survey,zhao2019uplift} or continuous doses \cite{hirano2005propensity,kennedy2017nonparametric,williams2020causal}. However, these formulations typically model treatments as separate arms or a one-dimensional dose, rather than as a coupled composition.
At the intersection of causal estimation and resource allocation, existing methods fall short of our setting. Budget-constrained uplift methods \cite{albert2022ecommerce,ai2022lbcf,sun2024end} allocate scarce resources across users via knapsack optimization, but restrict each user to a scalar or discrete treatment. Multi-channel advertising systems \cite{deng2023multichannel,shen2023crosschannel} coordinate campaign-level budgets, treating channel responses as aggregate primitives. Multi-touch attribution \cite{geyik2015multitouch,kumar2020camta} assigns conversion credit across sequences but lacks explicit simplex-constrained optimization. In contrast, our setting requires learning \textit{item-level} causal substitution under a fixed budget, where increasing one channel's allocation necessarily decreases another's, demanding a fundamentally different compositional formulation.

\subsection{Decision-Focused Optimization}
Offline policy learning evaluates policies from logged data using inverse-propensity or doubly robust estimators \cite{dudik2011doubly,dudik2014doubly,athey2021policy}. In industrial pipelines, the dominant paradigm is PTO, which decouples reward prediction from downstream optimization. To bridge the prediction-decision gap, Decision-Focused Learning (DFL) and SPO-style losses directly train predictive models against downstream decision metrics \cite{elmachtoub2022smart,wilder2019decision,sadana2025survey}. 
While PTO and DFL successfully align prediction with decision objectives in deterministic settings, they are fundamentally ill-equipped for causal compositional allocation. 
First, Factual Objective vs. Counterfactual Gradients. DFL optimizes factual reconstruction end-to-end. However, deployment strictly requires correct counterfactual marginal ranking. Differentiating through the optimizer with factual losses does not guarantee accurate causal gradients, especially in high-uplift regions.  Second, Lack of Extrapolation Conservatism. Global optimizers over learned black-box surfaces are notoriously aggressive. Unlike Conservative Q-Learning in offline RL \cite{kumar2020conservative}, standard DFL lacks mechanisms to penalize out-of-support predictions. On the constrained simplex, DFL actively exploits extrapolation errors, yielding disastrous out-of-distribution policies.  Third, Confounding Amplification. Historical allocations are heavily entangled with item states via legacy policies. Because DFL aggressively optimizes these biased factual predictions end-to-end without explicit causal orthogonalization, it amplifies unreliable causal gradients \cite{chernozhukov2018double,nie2021quasi}, leading to severe policy degradation.
\section{Problem Formulation}
\label{sec:problem}

\subsection{Setup}
Each decision has a fixed budget across \(K\) channels. Let \(X=(H,B)\) collect the pre-decision state \(H\in\mathcal H\) and the available budget \(B>0\). We observe historical decisions:
\begin{equation}
    \mathcal D
    =\{(X_i,P_i,Y_i)\}_{i=1}^{n},
    \qquad
    P_i\in\Delta^{K-1}
    :=\left\{p\in\mathbb R_{+}^{K}:\mathbf 1^\top p=1\right\},
    \label{eq:logged_data}
\end{equation}
where $P_i$ is the allocation proportion and $Y_i$ is the realized outcome. Uppercase $P_i$ denotes a historical action, whereas lowercase $p$ denotes a candidate decision.
Writing $Y(p)\equiv Y(Bp)$ for the potential outcome under allocation $p$, define
\(\mu(X,p):=\mathbb E[Y(p)\mid X]\) and \(V(\pi):=\mathbb E_X[\mu(X,\pi(X))]\).
Given a policy class $\Pi$, our objective is
\begin{equation}
    \pi^\star\in\arg\max_{\pi\in\Pi}V(\pi),
    \qquad
    \pi:X\mapsto\Delta^{K-1}.
    \label{eq:objective}
\end{equation}

\subsection{Local Reallocation Is the Decision Primitive}
Fixed budget makes every feasible first-order perturbation zero-sum. The tangent space of the simplex and its orthogonal projector are:
\begin{equation}
    \mathcal T
    :=\{v\in\mathbb R^K:\mathbf 1^\top v=0\},
    \qquad
    \Pi_{\mathcal T}
    :=I_K-\frac{1}{K}\mathbf 1\mathbf 1^\top.
    \label{eq:tangent_projection}
\end{equation}
For channels $k\neq\ell$, an infinitesimal transfer from $k$ to $\ell$ has local effect
\begin{align}
    D_{k\to\ell}\mu(X,p)
    &=\left.
      \frac{\mathrm d}{\mathrm d\delta}
      \mu\!\left(X,p+\delta(e_\ell-e_k)\right)
      \right|_{\delta=0^+}
      \nonumber\\
    &=(e_\ell-e_k)^\top\nabla_p\mu(X,p)
      =(e_\ell-e_k)^\top g^\star(X,p),
    \label{eq:pairwise_reallocation}
\end{align}
where \(g^\star(X,p)=\Pi_{\mathcal T}\nabla_p\mu(X,p)\) is the causal reallocation field. The policy signal is therefore not an absolute channel effect, but the relative marginal return from moving budget between channels.

\subsection{Production Failure Modes of PTO}
A standard industrial solution to~\eqref{eq:objective} is
PTO:
\begin{align}
    &\widehat\theta
    \in\arg\min_\theta
      \frac{1}{n}\sum_{i=1}^{n}
      \ell\!\left(Y_i,\widehat\mu_\theta(X_i,P_i)\right),
      \label{eq:pto_fit}\\
    &\widehat\pi_{\mathrm{PTO}}(X)
    \in\arg\max_{p\in\Delta^{K-1}}
      \widehat\mu_{\widehat\theta}(X,p).
      \label{eq:pto_pipeline}
\end{align}
Despite strong factual accuracy, this pipeline faces three failures. \textbf{Observational assignment:} historical allocations are selected by legacy policies, so causal interpretation requires adjustment. \textbf{Objective mismatch:} factual prediction loss does not control the marginal slopes or candidate rankings consumed by the optimizer. \textbf{Support mismatch:} global optimization may exploit response estimates outside the logged action region; These failures motivate the three stages of \textsc{ReAlloc}: orthogonal local estimation, marginal gradient distillation, and support-aware decision.

\section{Method}
\label{sec:method}

We propose \textsc{ReAlloc}, a casual teacher-student framework for fixed-budget multi-channel uplift learning; see Fig.~\ref{fig:method}. To handle the inherent non-stationarity of marketing environments and the vulnerabilities of standard PTO, \textsc{ReAlloc} operates as a \textbf{dual system}. At each update round, a \textbf{fast teacher} is trained on recent logs to extract unbiased local causal geometry, while a \textbf{slow student} accumulates this through a replay buffer to produce stable decisions. As detailed below, the three stages of \textsc{ReAlloc} systematically address the challenges of observational confounding, prediction-decision mismatch, and extrapolation vulnerability; see Algorithm~\ref{alg:realloc}.

\begin{figure}[t]
    \centering
    \includegraphics[width=\linewidth]{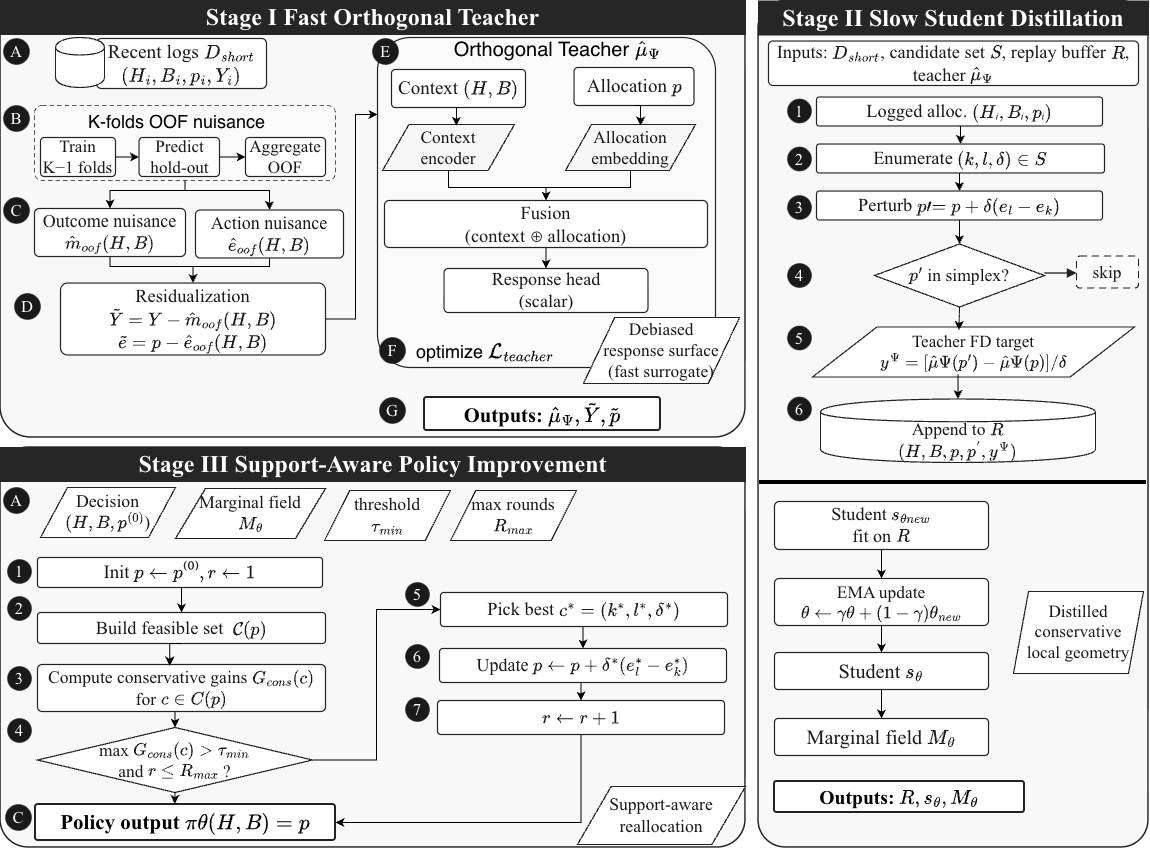}
    \caption{Overview of \textsc{ReAlloc}}
    \label{fig:method}
\end{figure}

\subsection{Stage I: Orthogonal Response Teacher}
\label{subsec:teacher}

To address \textit{observational confounding} (Challenge 1), we use orthogonalization, akin to Robinson's transformation in double machine learning (DML). For each recent training window $\mathcal D_{\mathrm{short}}$, we first estimate two nuisance functions via cross-fitting:
\begin{equation}
    \hat m(H,B) \approx \mathbb E[Y \mid H,B],
    \qquad
    \hat e(H,B) \approx \mathbb E[p \mid H,B],
    \label{eq:teacher-nuisance}
\end{equation}
where $\hat m$ captures the baseline demand and $\hat e \in \Delta^{K-1}$ captures the conditional mean action, i.e., the legacy policy tendency. We then construct the residualized outcome and treatment:
\begin{equation}
    \tilde Y_i = Y_i - \hat m(H_i,B_i),
    \qquad
    \tilde p_i = p_i - \hat e(H_i,B_i).
    \label{eq:teacher-residuals}
\end{equation}
Constrained by simplex the simplex, $\tilde p_i$ is a valid tangent-space deviation from the historical allocation, with the predictable component of treatment assignment removed. The teacher's response surface is parameterized as:
\begin{equation}
    \hat\mu_\psi(H,B,p)
    =
    \hat m(H,B)
    +
    r_\psi\bigl(H,B,p-\hat e(H,B)\bigr),
    \label{eq:teacher-surface}
\end{equation}
where $r_\psi$ is a residual response branch anchored at $r_\psi(H,B,0)=0$, ensuring that it models only the outcome variation driven by allocation deviations. Let
$\mathcal P_0 = I - \frac{1}{K}\mathbf 1\mathbf 1^\top$
be the projection onto the tangent space. The teacher's local marginal field is
\begin{equation}
    g_\psi(H,B)
    =
    \left.
    \mathcal P_0 \nabla_z r_\psi(H,B,z)
    \right|_{z=0}.
    \label{eq:teacher-gradient}
\end{equation}
We train the teacher by minimizing the composite objective
\begin{equation}
\begin{aligned}
    \mathcal L_{\mathrm{teacher}}
    =
    \frac{1}{|\mathcal D_{\mathrm{short}}|}
    \sum_{i\in\mathcal D_{\mathrm{short}}}
    \Bigl[
    &\ell\bigl(Y_i,\hat\mu_\psi(H_i,B_i,p_i)\bigr)
    + \lambda_{\mathrm{res}}
      \ell\bigl(\tilde Y_i,r_\psi(H_i,B_i,\tilde p_i)\bigr)
    \\
    &+ \lambda_{\mathrm{grad}}
      \ell\bigl(
        \tilde Y_i,
        \langle g_\psi(H_i,B_i),\tilde p_i\rangle
      \bigr)
    \Bigr].
    \label{eq:teacher-loss}
\end{aligned}
\end{equation}
The first term maintains factual predictive accuracy, the second forces the residual branch to explain residual outcome variation, and the third explicitly regularizes the local directional derivative, yielding reliable causal gradients in supported neighborhoods.

\subsection{Stage II: Student Marginal Distillation}
\label{subsec:student}

While the fast teacher adapts to recent dynamics, directly optimizing its response surface remains prone to the PTO failure mode (Challenge 2). We therefore train a slow student to distill the teacher's local geometry into a globally consistent marginal field. To ensure that the learned marginal field is integrable, i.e., path-consistent, we define the student through a scalar potential function $s_\theta(H,B,p)\in\mathbb R$. Its projected marginal utility is
\begin{equation}
    u_\theta(H,B,p)
    =
    \mathcal P_0 \nabla_p s_\theta(H,B,p).
    \label{eq:student-utility}
\end{equation}
The predicted local marginal gain for reallocating budget from channel $k$ to channel $l$ is
\begin{equation}
    M_{\theta,k\rightarrow l}(H,B,p)
    =
    u_{\theta,l}(H,B,p)-u_{\theta,k}(H,B,p).
    \label{eq:student-pairwise}
\end{equation}
This parameterization guarantees that finite allocation changes can be scored consistently through potential differences, avoiding cyclic or contradictory gradients.

At each round, the teacher generates supervision targets. For a feasible perturbation
$p' = p+\delta(e_l-e_k)$,
the finite-difference target is
\begin{equation}
    y^\psi_{k\rightarrow l,\delta}
    =
    \frac{
        \hat\mu_\psi(H,B,p')-\hat\mu_\psi(H,B,p)
    }{\delta}.
    \label{eq:teacher-finite-difference}
\end{equation}
We also extract the teacher's projected Jacobian
$g^\psi=\mathcal P_0\nabla_p\hat\mu_\psi(H,B,p)$.
The student is optimized over a replay buffer $\mathcal R$ of these targets:
\begin{equation}
\begin{aligned}
    \mathcal L_{\mathrm{student}}
    =
    \mathbb E_{\mathcal R}
    \Bigl[
    &\lambda_{\mathrm{pair}}
    \bigl(
        M_{\theta,k\rightarrow l}
        -y^\psi_{k\rightarrow l,\delta}
    \bigr)^2
    \\
    &+\lambda_{\mathrm{jac}}
    \left\|u_\theta-g^\psi\right\|_2^2
    \Bigr].
    \label{eq:student-loss}
\end{aligned}
\end{equation}
By accumulating targets in $\mathcal R$ and updating the deployed student via exponential moving average,
$\theta^{(t)}=\gamma\theta^{(t-1)}+(1-\gamma)\theta_{\mathrm{new}}$,
the student learns a stable, decision-focused marginal geometry.

\subsection{Stage III: Support-Aware Decision}
\label{subsec:decision}

To prevent \textit{extrapolation vulnerability} (Challenge 3), \textsc{ReAlloc} abandons unconstrained global maximization and instead performs conservative local reallocations. Given the current allocation $p$, we enumerate feasible local steps
\begin{equation}
    \mathcal C(p)
    =
    \left\{
        (k,l,\delta):
        p+\delta(e_l-e_k)\in\Delta^{K-1}
    \right\}.
    \label{eq:candidate-set}
\end{equation}
For a candidate $c=(k,l,\delta)$, let
$p_c=p+\delta(e_l-e_k)$.
The expected gain is computed from the student's potential difference:
\begin{equation}
    \widehat G_c
    =
    s_\theta(H,B,p_c)-s_\theta(H,B,p).
    \label{eq:potential-gain}
\end{equation}
To penalize unreliable updates, we define the conservative gain
\begin{equation}
    \widehat G_c^{\mathrm{cons}}
    =
    \widehat G_c
    -\beta\widehat\sigma_c
    -\lambda_s\varphi(\omega_c),
    \qquad
    \varphi(\omega_c)=-\log(\omega_c+\epsilon).
    \label{eq:conservative-gain}
\end{equation}
Here, $\widehat\sigma_c$ is the predictive uncertainty, estimated using ensemble variance or MC Dropout, and $\omega_c\in[0,1]$ is a directional support score, computed using kernel density or inverse KNN distance to the replay buffer $\mathcal R$. The latter measures whether similar reallocations are empirically supported by historical data. The policy selects the optimal conservative candidate:
\begin{equation}
    c^\star
    =
    \arg\max_{c\in\mathcal C(p)}
    \widehat G_c^{\mathrm{cons}}.
    \label{eq:conservative-candidate}
\end{equation}
The allocation is updated only if the conservative gain exceeds a safety threshold $\tau_{\min}$; otherwise, a no-op is executed. This localized, support-regularized search avoids unsupported, high-risk regions of the simplex while enabling conservative policy improvement.

\begin{algorithm}[t]
\caption{Implementation of \textsc{ReAlloc}}
\label{alg:realloc}
\SetAlgoLined
\KwIn{Recent logs $\mathcal D_{\mathrm{short}}$, replay buffer $\mathcal R$, candidate set $\mathcal S$, threshold $\tau_{\min}$, EMA rate $\gamma$}
\KwOut{Potential student $s_\theta$, marginal field $M_\theta$, policy $\pi_\theta$}

\textbf{Stage I: Fast Orthogonal Teacher}\;
Estimate nuisances $\hat m(H,B)$ and $\hat e(H,B)$ on $\mathcal D_{\mathrm{short}}$\;
Compute residuals $\tilde Y_i$ and $\tilde p_i$\;
Train teacher $\hat\mu_\psi$ by minimizing $\mathcal L_{\mathrm{teacher}}$\;

\textbf{Stage II: Slow Student Distillation}\;
\ForEach{$(H_i,B_i,p_i)\in\mathcal D_{\mathrm{short}}$}{
    \ForEach{$(k,l,\delta)\in\mathcal S$}{
        $p_i'\leftarrow p_i+\delta(e_l-e_k)$\;
        \If{$p_i'\in\Delta^{K-1}$}{
            $y^\psi_{i,k\rightarrow l,\delta}
            \leftarrow
            \dfrac{
                \hat\mu_\psi(H_i,B_i,p_i')
                -\hat\mu_\psi(H_i,B_i,p_i)
            }{\delta}$\;
            Add $(H_i,B_i,p_i,k,l,\delta,y^\psi_{i,k\rightarrow l,\delta})$ to $\mathcal R$\;
        }
    }
}
Update potential student $s_{\theta_{\mathrm{new}}}$ on $\mathcal R$ by minimizing $\mathcal L_{\mathrm{student}}$\;
$\theta\leftarrow\gamma\theta+(1-\gamma)\theta_{\mathrm{new}}$\;

\textbf{Stage III: Support-Aware Policy Improvement}\;
\ForEach{decision instance $(H,B,p^{(0)})$}{
    $p\leftarrow p^{(0)}$; $r\leftarrow1$\;
    \While{$r\le R_{\max}$ and $\max_c\widehat G_c^{\mathrm{cons}}(H,B,p)>\tau_{\min}$}{
        $c^\star\leftarrow
        \arg\max_{c\in\mathcal C(p)}
        \widehat G_c^{\mathrm{cons}}(H,B,p)$
        \tcp*[r]{Best conservative candidate}
        $p\leftarrow p+\delta^\star(e_{l^\star}-e_{k^\star})$; $r\leftarrow r+1$\;
    }
    $\pi_\theta(H,B)\leftarrow p$\;
}
\KwRet{$s_\theta$, $M_\theta$, $\pi_\theta$}\;
\end{algorithm}

\section{Theory}
\label{sec:theory}

Let $p_0(X)$ denote the incumbent allocation and let $\mathcal C(X)\subseteq\Delta^{K-1}$ denote the supported region. For an absolutely continuous path $\gamma:[0,1]\to\Delta^{K-1}$, define the path length:
\begin{equation}
    L(\gamma):=\int_0^1\|\dot\gamma(t)\|_2\,\mathrm dt.
    \label{eq:path_length}
\end{equation}
The path is support-admissible if
$\gamma(t)\in\mathcal C(X)$ for all $t\in[0,1]$. Define the \textbf{Pointwise Uplift} \(\Delta\) and \textbf{Uplift} \(\Gamma\):
\begin{align}
    \Delta(X,p)
    &:=\mu(X,p)-\mu(X,p_0(X)),
    \label{eq:pointwise_uplift}\\
    \Gamma(\pi)
    &:=\mathbb E_X[\Delta(X,\pi(X))],
    \qquad
    \mathcal R:=\{(X,p):p\in\mathcal C(X)\}.
    \label{eq:uplift_definitions}
\end{align}
Formal assumptions and proofs are deferred to
Appendix~\ref{app:theory_proofs}.

\subsection{Global Uplift Accumulates Local Return}
\label{subsec:theory_simplex}

\begin{theorem}[Simplex uplift along a feasible path]
\label{thm:simplex_integral}
Let $\gamma$ be an absolutely continuous path with
$\gamma(0)=p_0(X)$ and $\gamma(1)=p$. If $\mu(X,\cdot)$ is continuously
differentiable along $\gamma$, then
\begin{equation}
    \Delta(X,p)
    =\int_0^1
      g^\star(X,\gamma(t))^\top\dot\gamma(t)\,\mathrm dt.
    \label{eq:uplift_integral}
\end{equation}
\end{theorem}

Finite business value is the accumulated marginal return along the allocation moves that the system actually executes. ReAlloc therefore concentrates estimation capacity on supported local traces rather than reconstructing a globally response surface. Crucially, this integral equivalence motivates us to parameterize the student as a scalar potential function, ensuring that its utility difference between allocations exactly equals the path integral of local returns, thus serving as a direct and reliable proxy for the true uplift. The proof is given in Appendix~\ref{app:proof_simplex_integral}.

\subsection{Factual Fit Does Not Certify a Policy}
\label{subsec:theory_pto}

\begin{proposition}[Factual risk does not control decision quality]
\label{prop:pto_failure}
There exist smooth response surfaces and PTO predictors with zero factual risk under the logging distribution but constant decision regret. Moreover, there exists a sequence $\widehat\mu_n$ such that
$\|\widehat\mu_n-\mu\|_{L_2}\to0$ while
\(
    \left\|
      \Pi_{\mathcal T}\nabla_p\widehat\mu_n
      -\Pi_{\mathcal T}\nabla_p\mu
    \right\|_\infty\to\infty.
\)
\end{proposition}
The constructions are given in Appendix~\ref{app:proof_pto_failure}. Let $m(X,p):=\mathbb E[Y\mid X,P=p]$. A factual response model targets the
observational field
\begin{align}
    g_{\mathrm{obs}}(X,p)
    &:=\Pi_{\mathcal T}\nabla_p m(X,p)
      =g^\star(X,p)+b_{\mathrm{sc}}(X,p),
      \label{eq:pto_observational_field}\\
    b_{\mathrm{sc}}(X,p)
    &:=\Pi_{\mathcal T}\nabla_p\{m(X,p)-\mu(X,p)\}.
      \label{eq:pto_gradient_decomposition}
\end{align}
The exact shortcut bias decomposition is given in Appendix~\ref{app:proof_shortcut}. It vanishes under conditional ignorability; orthogonalization controls nuisance-estimation sensitivity after identification, but does not recover omitted confounders.
Offline RMSE and calibration are not launch certificates for an allocation policy. An optimizer can amplify small slope errors, and observational selection can make a historically favored channel appear incrementally valuable. Predictive fit, local directional accuracy, support coverage, and policy value must therefore be validated separately.

\subsection{Orthogonal Estimation of Supported Local Effects}
\label{subsec:theory_teacher}

Let $d=K-1$ and let $Q\in\mathbb R^{K\times d}$ be an orthonormal basis of $\mathcal T$. With $z=Q^\top p$, $Z=Q^\top P$, and
$\widetilde Z=Z-z$, let $\mathbb E_{h,z}[\cdot\mid X]$ denote kernel localization around $z$. Define
\begin{equation}
    e_{h,z}(X):=\mathbb E_{h,z}[\widetilde Z\mid X],
    \qquad
    m_{h,z}(X):=\mathbb E_{h,z}[Y\mid X].
    \label{eq:local_nuisances}
\end{equation}
The orthogonal component of the teacher is characterized by
\begin{align}
    \Psi_{h,z}(\beta,m,e;X)
    :=\mathbb E_{h,z}\!\Big[&\{\widetilde Z-e(X)\}
      \{Y-m(X)\nonumber\\[-1mm]
      &-\beta(X,z)^\top(\widetilde Z-e(X))\}\,\Big|\,X\Big].
    \label{eq:orthogonal_moment}
\end{align}

\begin{theorem}[Orthogonal local-field estimation]
\label{thm:teacher_identification}
Under Assumptions~\ref{ass:identification}--\ref{ass:estimation}, the
population root $\beta_h^\dagger$ of~\eqref{eq:orthogonal_moment} satisfies
\begin{equation}
    \sup_{(X,p)\in\mathcal R}
    \left\|Q\beta_h^\dagger(X,Q^\top p)-g^\star(X,p)\right\|_2
    \le\epsilon_{\mathrm{loc}}(h),
    \label{eq:population_teacher_bias}
\end{equation}
and the score is Neyman orthogonal with respect to $(m,e)$. Let
$\widehat g_T^{\mathrm{orth}}:=Q\widehat\beta$ be the cross-fitted empirical
root and let $\widehat g_T$ be the field returned by the implemented composite
teacher. If
\begin{equation}
    \|\widehat g_T-\widehat g_T^{\mathrm{orth}}\|_{\infty,\mathcal R}
    \le\epsilon_{\mathrm{comp}},
    \label{eq:composite_bridge}
\end{equation}
then, with probability at least $1-\delta$,
\begin{align}
    \|\widehat g_T-g^\star\|_{\infty,\mathcal R}
    &\le\epsilon_T(n,h,\delta)+\epsilon_{\mathrm{comp}},
    \label{eq:teacher_error}\\
    \epsilon_T(n,h,\delta)
    &:=\epsilon_{\mathrm{loc}}(h)
      +C\sqrt{\frac{\mathfrak C_{\mathcal G}+\log(1/\delta)}
                     {nh^{K+1}}}
      +C r_{\mathrm{orth}}.
    \label{eq:teacher_rate}
\end{align}
\end{theorem}
The normalized second-order remainder $r_{\mathrm{orth}}$ and the proof are given in Appendix~\ref{app:proof_teacher_identification}. 
The rate exposes the production trade-off. A smaller neighborhood reduces local approximation bias but also reduces effective sample size; the $h^{K+1}$ term makes channel dimensionality an explicit data requirement. Orthogonality allows the outcome and logging models to be improved modularly, while $\epsilon_{\mathrm{comp}}$ records approximation and optimization error from the full neural objective.

\subsection{Potential Accuracy Controls Regret}
\label{subsec:theory_regret}
\label{subsec:theory_safe}

The deployed student is parameterized by a scalar potential:
\(\widehat g(X,p) :=\Pi_{\mathcal T}\nabla_p s_\theta(X,p)\), \(\widehat\Delta_\theta(X,p) :=s_\theta(X,p)-s_\theta(X,p_0(X))\).
For $p\in\mathcal C(X)$, let $\mathfrak P_X(p)$ denote the set of support-admissible paths from $p_0(X)$ to $p$, and define the reachable action set:
\begin{equation}
    \mathcal A_{L_{\max}}(X)
    :=\left\{p:\exists\gamma\in\mathfrak P_X(p),\ L(\gamma)\le L_{\max}\right\}.
    \label{eq:reachable_set}
\end{equation}
Let
$\pi_{L_{\max}}^\star(X)\in
\arg\max_{p\in\mathcal A_{L_{\max}}(X)}\mu(X,p)$ and define
\begin{equation}
    \operatorname{Reg}_{L_{\max}}(\widehat\pi)
    :=\mathbb E_X\!\left[
      \mu(X,\pi_{L_{\max}}^\star(X))
      -\mu(X,\widehat\pi(X))
    \right].
    \label{eq:trace_regret}
\end{equation}
Assume the local search returns $\widehat\pi(X)\in\mathcal A_{L_{\max}}(X)$
and satisfies
\begin{equation}
    \mathbb E_X\!\left[
      \sup_{p\in\mathcal A_{L_{\max}}(X)}
      \widehat\Delta_\theta(X,p)
      -\widehat\Delta_\theta(X,\widehat\pi(X))
    \right]
    \le\epsilon_{\mathrm{search}}.
    \label{eq:search_error}
\end{equation}

\begin{theorem}[Trace-wise field error controls value and regret]
\label{thm:field_regret}
Suppose
\(
    \|\widehat g-g^\star\|_{\infty,\mathcal R}\le\epsilon_g.
\)
For any $p\in\mathcal A_{L_{\max}}(X)$ and any
$\gamma\in\mathfrak P_X(p)$,
\begin{align}
    \widehat\Delta_\theta(X,p)
    &=\int_0^1
      \widehat g(X,\gamma(t))^\top\dot\gamma(t)\,\mathrm dt,
    \label{eq:potential_line_integral}\\
    \left|\widehat\Delta_\theta(X,p)-\Delta(X,p)\right|
    &\le L(\gamma)\epsilon_g.
    \label{eq:uplift_error_bound}
\end{align}
Consequently,
\begin{equation}
    \operatorname{Reg}_{L_{\max}}(\widehat\pi)
    \le2L_{\max}\epsilon_g+\epsilon_{\mathrm{search}}.
    \label{eq:field_regret_bound}
\end{equation}
\end{theorem}

\begin{corollary}[End-to-end teacher--student regret]
\label{cor:teacher_student_regret}
Under Assumption~\ref{ass:approximation} and the event in
Theorem~\ref{thm:teacher_identification},
\begin{equation}
    \operatorname{Reg}_{L_{\max}}(\widehat\pi)
    \le
    2L_{\max}
    \{\epsilon_T(n,h,\delta)+\epsilon_{\mathrm{comp}}+\epsilon_S\}
    +\epsilon_{\mathrm{search}}.
    \label{eq:teacher_student_regret}
\end{equation}
\end{corollary}
The proofs are given in Appendix~\ref{app:proof_field_regret}.
$L_{\max}$ is a deployment blast-radius control: larger cumulative budget movement creates more upside but also amplifies field error. The term $\epsilon_{\mathrm{search}}$ captures the value lost to finite candidate sets, greedy search, and latency constraints. Because the student is a scalar potential, gains telescope across the actual accepted greedy trace and remain path-consistent.
\section{Synthetic Experiment}
\label{sec:synthetic}

\subsection{Data Generation Process (DGP)}
\label{subsec:synthetic_dgp}

We construct a synthetic environment with known counterfactual outcomes to evaluate \textit{ReAlloc}. The DGP is designed to preserve three core challenges: state-dependent assignment, locally identifiable but globally under-supported actions, and complex channel interactions. Each observation is an item-period tuple $(H_{it},B_{it},P_{it},Y_{it})$, where the allocation $P_{it}\in\Delta^{K-1}$ lies on a simplex with $K=3$ channels.

\noindent\textbf{State-dependent logging.} 
We decompose the state as $H=(C, E, S)$, where $C$ acts as the primary confounder, $E$ modifies the response, and $S$ captures channel preference. To handle compositional allocations $P$, we map the simplex to a vector space using the isometric log-ratio (ILR) transform. The logging policy is specifically designed to induce three critical properties:
(1) \textbf{State-dependent confounding}, controlled by a parameter $\alpha_{\mathrm{cf}}$, which couples the primary confounder $C$ with the allocation; 
(2) \textbf{Boundary skewness}, controlled by $\alpha_{\mathrm{bd}}$, which pushes the propensity mean toward simplex boundaries to mimic extreme budget skews; and 
(3) \textbf{Strict overlap}, achieved via a logistic-normal mixture noise that balances local exploitation (low variance) with sparse global exploration. 
Crucially, since all confounders are fully observed in $H$, the DGP strictly satisfies conditional ignorability. \textit{(See Appendix~\ref{app:dgp_details}.)}

\noindent\textbf{Joint response surface.}
Outcomes follow $Y = \mu_\star(H,B,P)+\epsilon$ with $\epsilon\sim\mathcal N(0,\sigma_Y^2)$. The causal response is anchored at a state-independent baseline allocation $c(H,B)$. Let $z$ denote the allocation deviation from the baseline in the ILR space. The response surface:
\begin{equation}
\begin{aligned}
    &\mu_\star(H,B,P) = m_0(H,B) 
    \\
    &+ \rho(H,B) \Big[
    \omega_{\mathrm{loc}}g_{\mathrm{loc}}(H,z) 
    + \omega_{\mathrm{int}}g_{\mathrm{int}}(H,z) 
    + \omega_{\mathrm{far}}g_{\mathrm{far}}(H,z) \Big].
\label{eq:dgp_response}
\end{aligned}
\end{equation}

This formulation explicitly disentangles the ROI curve into three regimes:
(1) The \textbf{local term} $g_{\mathrm{loc}}(H,z)$ captures linear marginal returns, providing an identifiable first-order reallocation signal.
(2) The \textbf{interaction term} $g_{\mathrm{int}}$ models cross-channel substitution and complementarity.
(3) The \textbf{far-field term} $g_{\mathrm{far}}$ is flat near the anchor but introduces non-linear saturation and cannibalization for out-of-support (OOS) allocations.
The component scales ($\omega$) are calibrated to ensure local gradients remain learnable from logs, while distant curvature is weakly identified, rigorously testing the model's extrapolation capability.

\noindent\textbf{Temporal variations.}
For the temporal experiment, we keep $\mu_\star$ and all response parameters fixed, and apply a smooth window-specific bias to the logging logits. This rotates the observed action support across channels without introducing sequential treatment effects. Hence, we evaluate the retention of local geometric knowledge under changing support, rather than mechanism drift.

\subsection{Setup}
\label{subsec:synthetic_setup}

We design synthetic experiments to address three questions: \textbf{RQ1 (Safe Policy Improvement):} Under increasingly severe confounding and low overlap, can \textsc{ReAlloc} achieve higher deployable uplift than baselines?
\textbf{RQ2 (Mechanism of Improvement):} Which components drive the performance gains?
\textbf{RQ3 (Fast-Slow Temporal Memory):} Can the slow student effectively accumulate and retain local geometric knowledge from successive teachers?

\noindent\textbf{Baselines.} 
We compare \textsc{ReAlloc} against a comprehensive suite of reference, industry-style, and causal baselines:
(1) \textit{Logging} (maintains historical allocations) and \textit{Uniform} (equal budget split) serve as reference policies. (2) \textit{Additive ROI} models channel responses independently, reflecting the decoupled paradigm in industry.
(3) \textit{Joint S-Learner PTO} fits a global response surface using a standard S-learner and applies global or local optimization.
(4) \textit{R-Learner Local} orthogonally residualizes outcomes and treatments to estimate heterogeneous effects, selecting the optimal locally.

\noindent\textbf{Evaluation Metrics.}
Standard offline policy evaluation often ignores the risk of OOD recommendations. We introduce a fallback mechanism: if a policy $\hat\pi$ recommends an allocation $\hat{p}_i$ that falls outside the historically supported region $\mathcal{S}_i$, the system safely defaults to the factual logging allocation $P_i$. 
Let $\Delta_i(\hat\pi) = \mu_\star(H_i,B_i,\hat p_i) - \mu_\star(H_i,B_i,P_i)$ be the pointwise oracle uplift. We define the \textbf{Deployable Uplift} as our primary metric:
\begin{equation}
\mathrm{Uplift}_{\mathrm{dep}}(\hat\pi)
=
\frac{1}{N}\sum_{i=1}^N
\mathbf{1}\{\hat{p}_i \in \mathcal{S}_i\} \Delta_i(\hat\pi),
\label{eq:deployable_uplift}
\end{equation}
To provide a comprehensive evaluation, we additionally report the \textsc{OOS rate} (the fraction of rejected recommendations), the \textsc{Safe Local Recovery} ratio (comparing $U_{\mathrm{dep}}$ against an oracle), and geometric ranking metrics (\textsc{EdgeNDCG}, \textsc{TopEdgeAcc}, \textsc{TopEdgeRegret}, \textsc{PairwiseCorr}) that evaluate the local directional accuracy independent of policy visitation.

\subsection{Results and Analysis}
\label{sec:sim-results}

\paragraph{RQ1 result.}
Table~\ref{tab:static-main} reports policy quality under decreasing logging overlap. Unconstrained PTO attains high raw uplift but converts little of it into deployable value, whereas support-constrained S-learners recover most of this loss, showing that unsafe global search is a major source of PTO failure. ReAlloc remains the strongest learned policy across all regimes, achieving deployable uplifts of $.97/.79/.71/.57$ and recovering $.88/.86/.84/.83$ of the local oracle opportunity; its margin over the strongest S-NN variant is $.04/.07/.09/.06$. Additive ROI and the orthogonal R-Learner remain feasible but obtain substantially lower value. Notably, the R-Learner slightly outperforms ReAlloc on hard-regime logged-anchor NDCG and regret, yet reaches only $.21$ deployable uplift versus $.71$, indicating that single-step edge quality alone does not determine effective multi-step reallocation.

\begin{table}[t]
\centering
\caption{
\textbf{Static policy quality under decreasing overlap.}
Cells show Deployable uplift / Safe Recovery (D/SR). 
For unconstrained PTO ($_\text{G}$), Raw uplift (R) is shown as superscript.
}
\label{tab:static-main}

\resizebox{\columnwidth}{!}{%
\footnotesize
\setlength{\tabcolsep}{3pt}
\renewcommand{\arraystretch}{1.1}
\begin{tabular}{@{}l cccc ccc@{}}
\toprule
& \multicolumn{4}{c}{Uplift (D/SR)} & \multicolumn{3}{c}{Hard Regime Diagnostics} \\
\cmidrule(lr){2-5} \cmidrule(l){6-8}
Method & Benign & Medium & Hard & Extreme & Bias & NDCG & Regret \\
\midrule
\multicolumn{8}{@{}l}{\textit{References}}\\[-2pt]
Logging & $.00/\text{--}$ & $.00/\text{--}$ & $.00/\text{--}$ & $.00/\text{--}$ & -- & -- & -- \\
Uniform & $-.01/\text{--}$ & $-.03/\text{--}$ & $-.03/\text{--}$ & $-.01/\text{--}$ & -- & -- & -- \\
Add.\ ROI & $.22/.20$ & $.21/.22$ & $.18/.22$ & $.16/.23$ & $-.02$ & $.89$ & $.78$ \\

\midrule
\multicolumn{8}{@{}l}{\textit{Unconstrained global PTO}}\\[-2pt]
S-GBDT$_\text{G}$ & $\mathbf{.08}^{1.0}/.07$ & $\mathbf{.03}^{.8}/.03$ & $\mathbf{.01}^{.6}/.01$ & $\mathbf{.02}^{.7}/.02$ & $+.12$ & $.70$ & $2.62$ \\
S-NN$_\text{G}$ & $\mathbf{.02}^{1.2}/.02$ & $\mathbf{.01}^{1.1}/.01$ & $\mathbf{.01}^{1.0}/.01$ & $\mathbf{.00}^{1.0}/.00$ & $-.15$ & $.89$ & $.70$ \\

\midrule
\multicolumn{8}{@{}l}{\textit{Support-constrained PTO}}\\[-2pt]
S-GBDT$_\text{C}$ & $.75/.68$ & $.46/.50$ & $.32/.39$ & $.36/.52$ & $+.03$ & $.70$ & $2.62$ \\
S-NN$_\text{C}$ & $.93/.84$ & $.72/.78$ & $.62/.74$ & $.50/.73$ & $-.12$ & $.89$ & $.70$ \\

\midrule
\multicolumn{8}{@{}l}{\textit{Shared local-support search}}\\[-2pt]
S-GBDT$_\text{L}$ & $.01/.01$ & $.01/.01$ & $.00/.00$ & $.00/.01$ & $+.00$ & $.70$ & $2.62$ \\
S-NN$_\text{L}$ & $.91/.82$ & $.71/.76$ & $.61/.73$ & $.51/.74$ & $-.11$ & $.89$ & $.70$ \\
R-Learner$_\text{L}$ & $.24/.21$ & $.22/.24$ & $.21/.25$ & $.18/.27$ & $-.03$ & $\mathbf{.92}$ & $\mathbf{.49}$ \\

\textbf{ReAlloc} & $\mathbf{.97}/.88$ & $\mathbf{.79}/.86$ & $\mathbf{.71}/.84$ & $\mathbf{.57}/.83$ & $+.05$ & $\underline{.92}$ & $\underline{.50}$ \\

\midrule
Oracle Local & $1.10/1.00$ & $.92/1.00$ & $.83/1.00$ & $.68/1.00$ & -- & -- & -- \\
\bottomrule
\end{tabular}%
} 
\vspace{2pt}
{\scriptsize Subscripts: G=global, C=constrained, L=local search.}
\end{table}

Figure~\ref{fig:exp1-static-stress} complements the aggregate results. For a fixed context and oracle response surface, increasing assignment severity concentrates the conditional logged actions, isolating support deterioration from changes in the underlying outcome function. The best-edge map further shows that the preferred local transfer depends on the current simplex position. The trajectory comparison illustrates how global PTO moves toward a distant unsupported action, whereas the shared local-support policies remain feasible and follow distinct reallocation paths.

\begin{figure}[t]
    \centering
    \captionsetup[subfigure]{
        font=scriptsize,
        labelfont=bf,
        justification=centering
    }
    \begin{subfigure}[t]{0.235\linewidth}
        \centering
        \includegraphics[width=\linewidth]
        {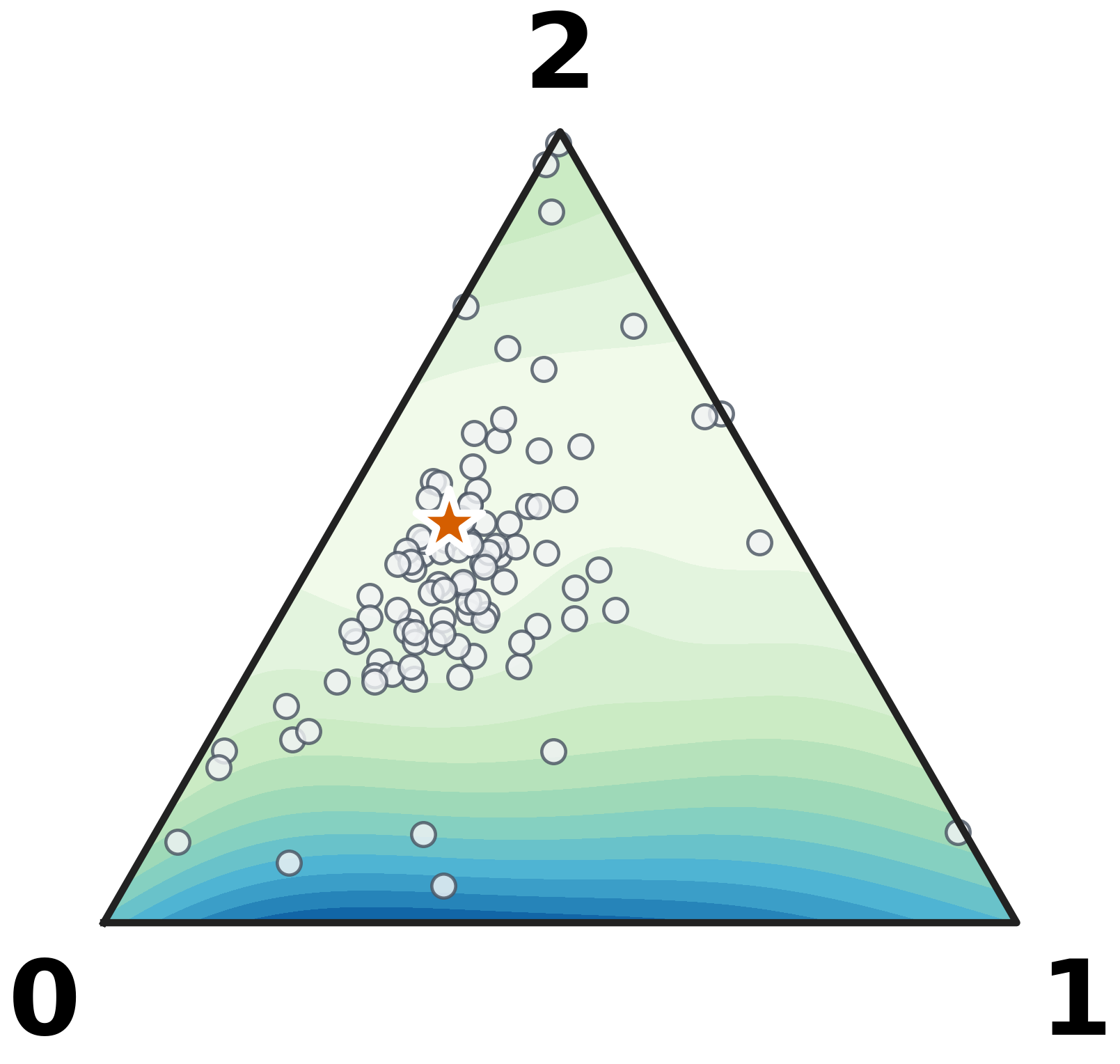}
        \caption{Benign}
        \label{fig:exp1-surface-benign}
    \end{subfigure}
    \hfill
    \begin{subfigure}[t]{0.235\linewidth}
        \centering
        \includegraphics[width=\linewidth]
        {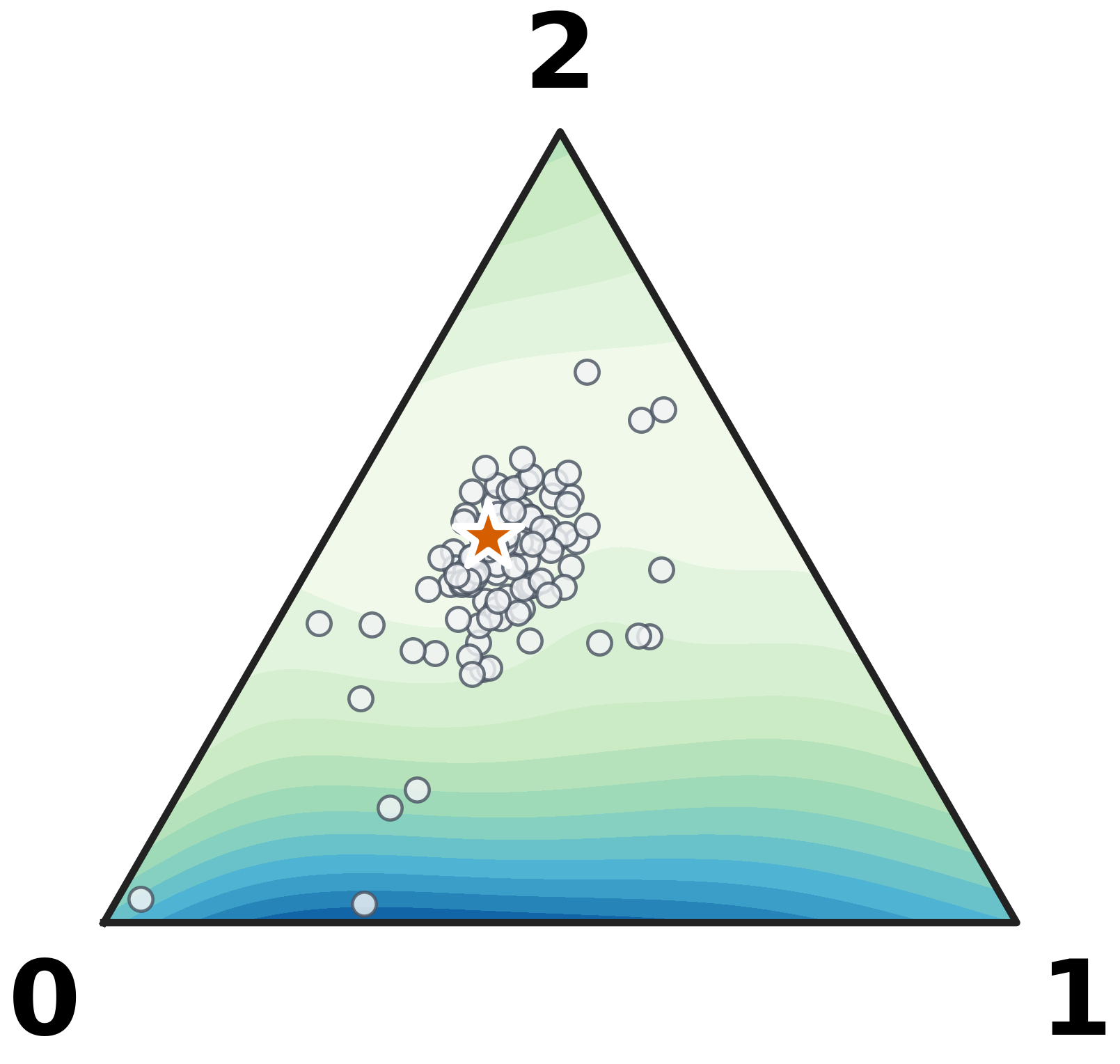}
        \caption{Medium}
        \label{fig:exp1-surface-medium}
    \end{subfigure}
    \hfill
    \begin{subfigure}[t]{0.235\linewidth}
        \centering
        \includegraphics[width=\linewidth]
        {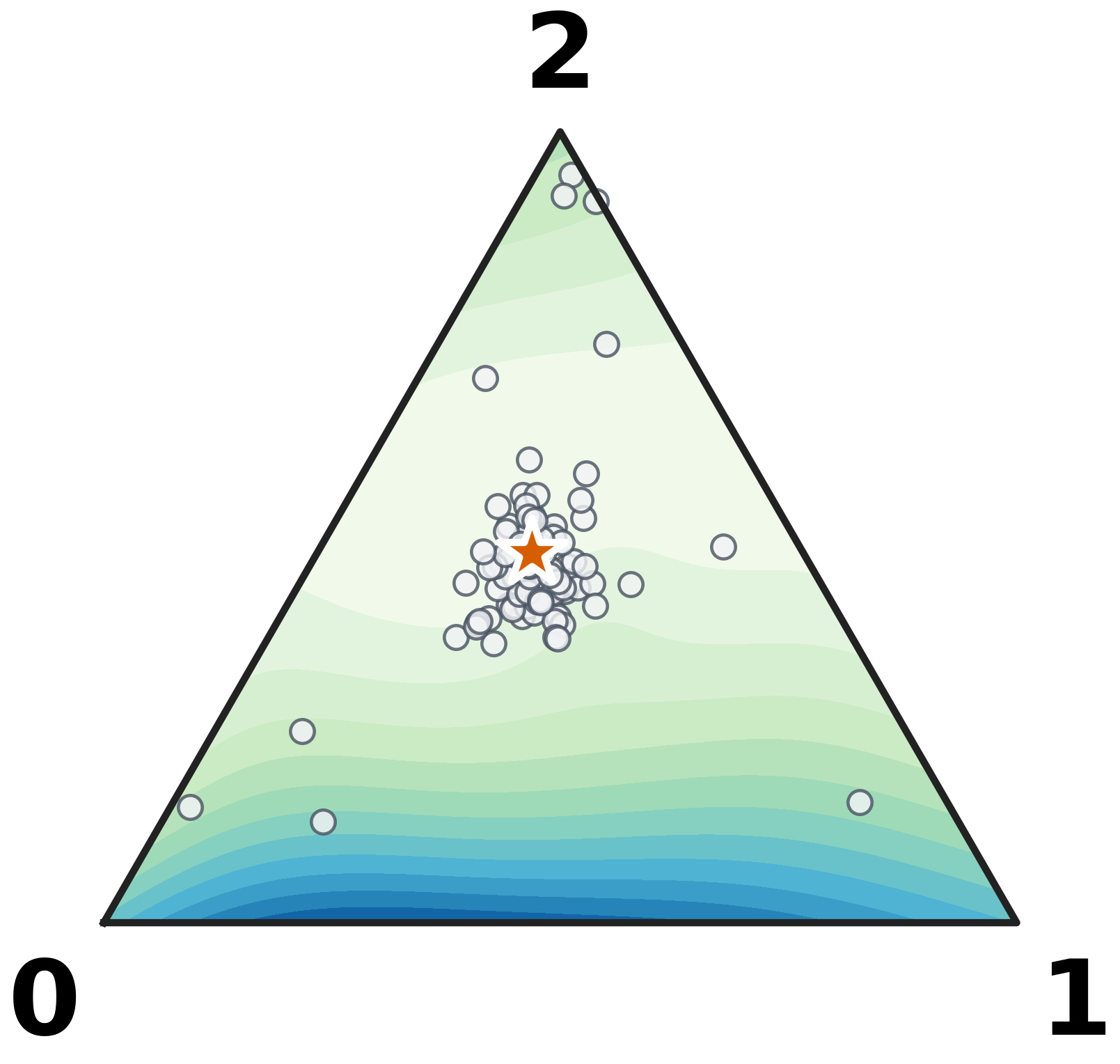}
        \caption{Hard}
        \label{fig:exp1-surface-hard}
    \end{subfigure}
    \hfill
    \begin{subfigure}[t]{0.235\linewidth}
        \centering
        \includegraphics[width=\linewidth]
        {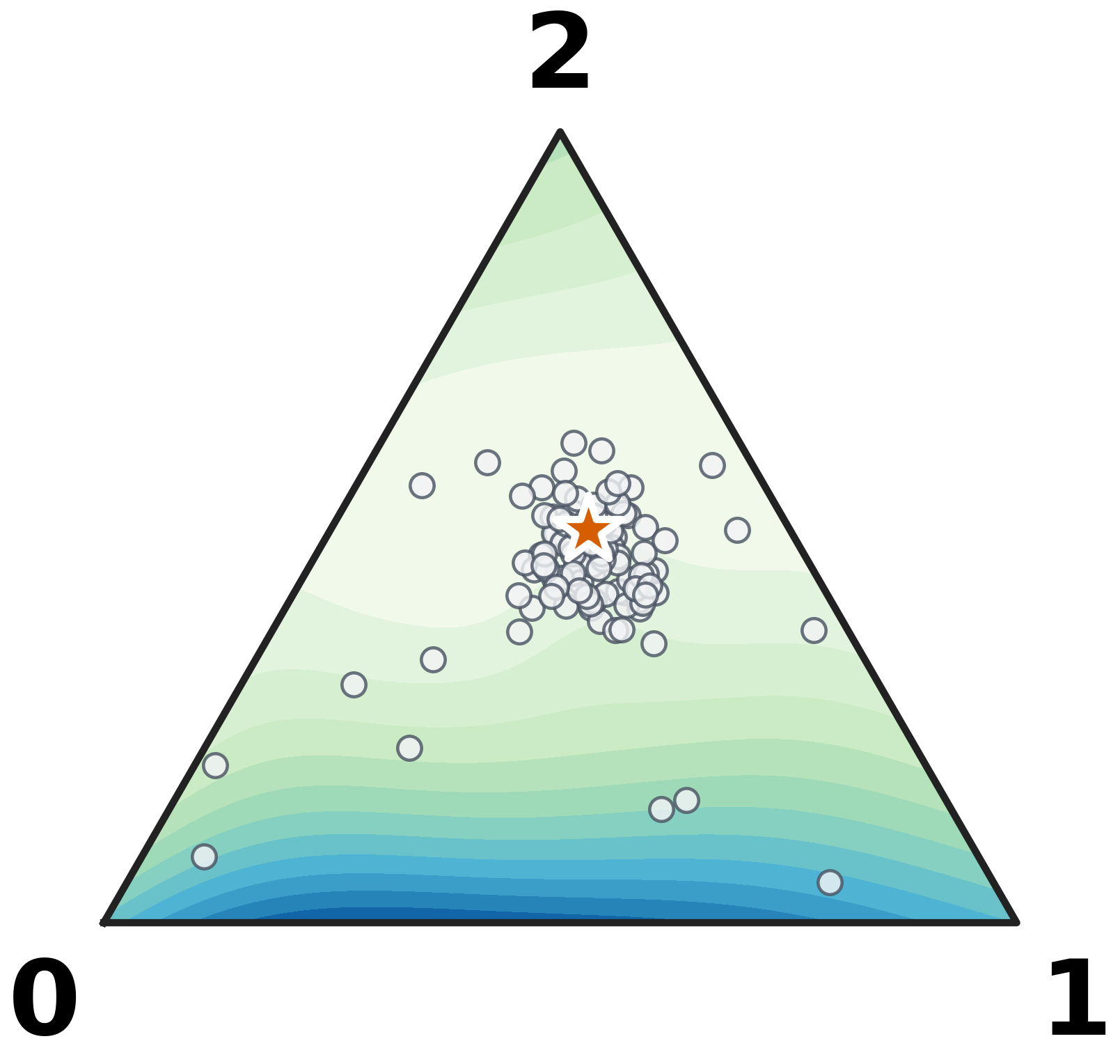}
        \caption{Extreme}
        \label{fig:exp1-surface-extreme}
    \end{subfigure}

    \begin{subfigure}[t]{0.485\linewidth}
        \centering
        \includegraphics[width=\linewidth]
        {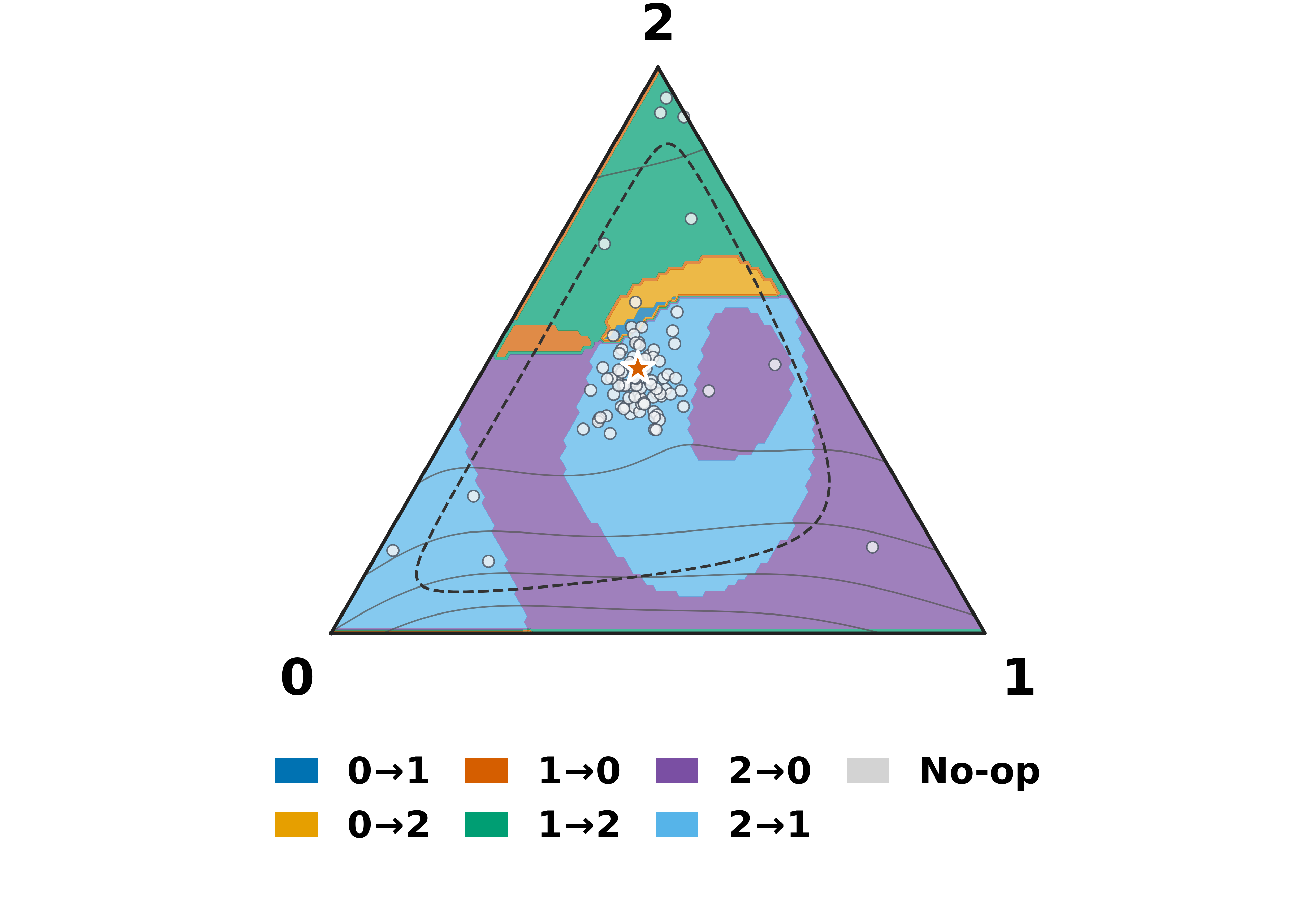}
        \caption{Oracle best local transfer}
        \label{fig:exp1-best-edge-hard}
    \end{subfigure}
    \hfill
    \begin{subfigure}[t]{0.485\linewidth}
        \centering
        \includegraphics[width=\linewidth]
        {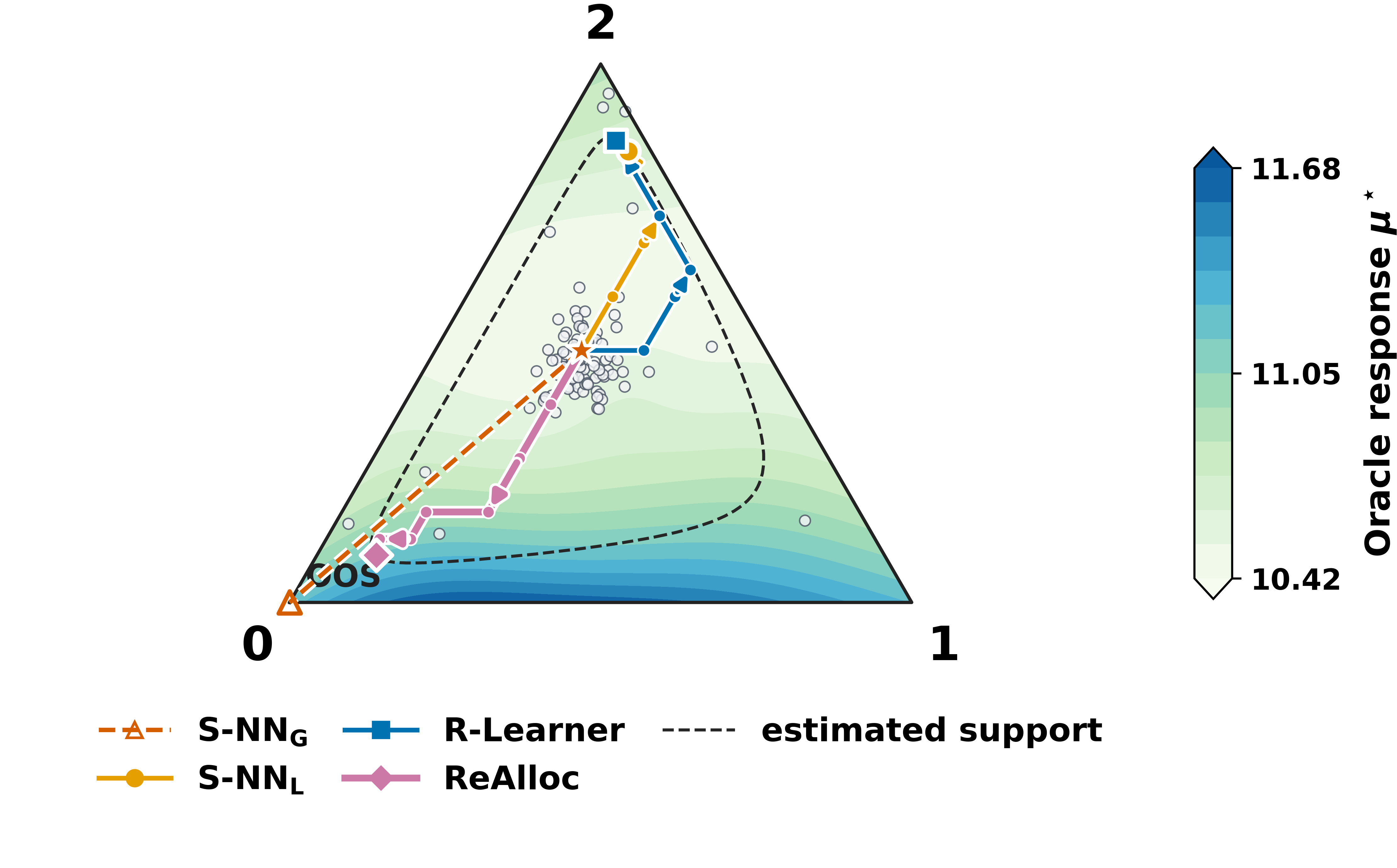}
        \caption{Representative policy trajectories}
        \label{fig:exp1-trajectory-hard}
    \end{subfigure}

    \caption{
        \textbf{Orcle response surface and policy behavior.}
    }
    \label{fig:exp1-static-stress}
\end{figure}

\paragraph{RQ2 result.}
Table~\ref{tab:sim-ablation} shows that ReAlloc requires both accurate local geometry and support-aware deployment.
The results separate geometry errors from safety errors.
Teacher-only local greedy remains support-safe, but it achieves weak edge correlation and recovers only a small fraction of the local oracle gain, showing that directly deploying noisy teacher targets is insufficient for stable policy improvement.
The most severe safety failure occurs when the support check is removed: despite using a learned local field, the policy frequently leaves the empirical support, leading to high OOS and a non-negligible raw-safe gap.
Overall, ReAlloc achieves the best combination of edge-ranking accuracy, deployable uplift, and support safety among learned variants, recovering about 60\% of the support-aware local oracle gain without accessing oracle counterfactuals.

\begin{table}[t]
\centering
\scriptsize
\setlength{\tabcolsep}{2.0pt}
\caption{
\textbf{Mechanism ablation under hard interaction stress.}
}
\label{tab:sim-ablation}
\resizebox{\columnwidth}{!}{%
\begin{tabular}{lcccccccc}
\toprule
\textbf{Variant}
& \multicolumn{4}{c}{\textbf{Local geometry}}
& \multicolumn{4}{c}{\textbf{Decision / safety}} \\
\cmidrule(lr){2-5}\cmidrule(lr){6-9}
& \textbf{Edge} $\uparrow$
& \textbf{Top} $\uparrow$
& \textbf{Corr} $\uparrow$
& \textbf{Err} $\downarrow$
& \textbf{Dep.} $\uparrow$
& \textbf{OOS} $\downarrow$
& \textbf{Gap} $\downarrow$
& \textbf{Rec.} $\uparrow$ \\
\midrule
Teacher-only
& .816 & .343 & .429 & 1.173
& .144 & .000 & .000 & .129 \\

w/o Orthogonalization
& .790 & .294 & .407 & .977
& .412 & .000 & .000 & .357 \\

w/o Support
& .784 & .291 & .407 & .922
& .001 & .984 & .589 & .001 \\

\textbf{ReAlloc}
& \textbf{.828} & \textbf{.370} & \textbf{.492} & 1.011
& \textbf{.522} & \textbf{.000} & \textbf{.000} & \textbf{.455} \\
\bottomrule
\end{tabular}%
}
\end{table}

\paragraph{RQ3 result.}
We finally evaluate periodic support rotation under a fixed response surface. With $N$ windows and $n$ rows per window, a \textit{Pooled Teacher} stores $\Theta(Nn)$ rows and incurs $\Theta(N^2n)$ cumulative fitting work by repeatedly retraining on the growing history. In contrast, \textsc{Fast-Slow ReAlloc} updates a current-window teacher alongside a fixed-memory student, requiring only $\Theta(n+B)$ storage and $\Theta(N(n+B))$ cumulative work, where $B$ is the student's buffer size. In our configuration ($B \approx n$, $N=16$), \textsc{Fast-Slow} uses only $12.5\%$ of the pooled storage and $36.9\%$ of its cumulative fitting time, while retaining $94.9\%$ of its final union uplift ($0.080$ vs.\ $0.084$). This bounded update cost becomes increasingly important when full-history retraining grows progressively more expensive.

\begin{table}[t]
\centering
\scriptsize
\setlength{\tabcolsep}{3.0pt}
\caption{
\textbf{Final rolling-window summary.}
}
\label{tab:sim-temporal}
\resizebox{\columnwidth}{!}{%
\begin{tabular}{lccccc}
\toprule
\textbf{Method}
& \textbf{Curr.} $\uparrow$
& \textbf{Past} $\uparrow$
& \textbf{Union} $\uparrow$
& \textbf{Memory} $\downarrow$
& \textbf{Time} $\downarrow$ \\
\midrule
Fresh Teacher
& .144 & .045 & .051
& $0.50\times$
& $0.34\times$ \\

Reset Student
& .152 & .047 & .053
& $0.50\times$
& $0.43\times$ \\

Warm-start Teacher
& .146 & .048 & .054
& $0.50\times$
& $0.32\times$ \\

Equal-memory Raw Replay
& .208 & .068 & .077
& $1.00\times$
& $0.95\times$ \\

\textbf{Fast-Slow ReAlloc}
& .216 & .071 & .080
& $1.00\times$
& $1.00\times$ \\
\midrule

Pooled Teacher
& \textbf{.224} & \textbf{.075} & \textbf{.084}
& $8.00\times$
& $2.71\times$ \\
\bottomrule
\end{tabular}%
}
\end{table}
\section{Real-World Evaluation on Taobao}
\label{sec:real_world}

We evaluate \textsc{ReAlloc} on a 60-day Taobao production dataset comprising 500K items, each allocating a fixed budget across paid advertising and promotional benefits. Evaluations are conducted via matched prospective replay and online A/B testing.

\subsection{Offline Evaluation}
\label{sec:real_offline}

\begin{table*}[t]
\centering
\caption{Offline evaluation on routine production traffic and randomized exploration traffic.}
\label{tab:real_offline}
\footnotesize
\setlength{\tabcolsep}{3.5pt}
\renewcommand{\arraystretch}{1.15}

\begin{tabular}{l cccc cccc}
\toprule
& \multicolumn{4}{c}{\textbf{A. Routine traffic}} & \multicolumn{4}{c}{\textbf{B. Randomized exploration traffic}} \\
\cmidrule(lr){2-5} \cmidrule(lr){6-9}
Method 
& $\rho(G,R){\uparrow}$ 
& Within-item conc.${\uparrow}$ 
& A--R gap${\uparrow}$ 
& Supp.\ viol.\,@\,upd.${\downarrow}$
& DR lift${\uparrow}$ 
& Action rate 
& Lift/act.${\uparrow}$ 
& Agreement lift${\uparrow}$ \\
\midrule
\textsc{ReAlloc} 
& $.028_{[.021,.036]}$ & $.512$ & $.023$ & $\mathbf{.000}@.31$ 
& $\mathbf{.025}_{[.012,.043]}$ & $.33$ & $\mathbf{.069}$ & $.098_{[.067,.153]}$ \\
\textsc{ReAlloc} w/o Supp. 
& $.028_{[.021,.036]}$ & $.512$ & $.023$ & $.678@.32$ 
& $-.021_{[-.028,-.012]}$ & $.96$ & $-.021$ & $-.040_{[-.061,-.022]}$ \\
\midrule
PTO-Local 
& $.010_{[.000,.019]}$ & $.507$ & $.012$ & $.675@.31$ 
& $-.019_{[-.026,-.011]}$ & $.97$ & $-.020$ & $-.025_{[-.044,-.004]}$ \\
Teacher-Only 
& $.017_{[.009,.024]}$ & $.509$ & $.011$ & $.663@.27$ 
& $-.019_{[-.028,-.011]}$ & $.92$ & $-.021$ & $-.024_{[-.045,-.002]}$ \\
Context-Only 
& $.019_{[.007,.025]}$ & $.506$ & $.007$ & $.696@.33$ 
& $-.015_{[-.023,-.006]}$ & $.97$ & $-.015$ & $-.022_{[-.043,-.001]}$ \\
\midrule
VCNet 
& $.029_{[.021,.039]}$ & $.514$ & $.026$ & $.660@.34$ 
& $.022_{[.015,.032]}$ & $.98$ & $.023$ & $.073_{[.055,.093]}$ \\
GIKS 
& $.026_{[.018,.035]}$ & $.513$ & $.023$ & $.657@.33$ 
& $.022_{[.014,.031]}$ & $.98$ & $.022$ & $.070_{[.051,.092]}$ \\
AdditiveROI 
& $\mathbf{.031}_{[.022,.040]}$ & $\mathbf{.516}$ & $\mathbf{.029}$ & $.575@.34$ 
& $.023_{[.015,.032]}$ & $.97$ & $.024$ & ${.078}_{[.059,.098]}$ \\
\bottomrule
\end{tabular}
\end{table*}

Offline evaluation combines two complementary data because routine production logs suffer from low overlap and unstable IPS estimates, as risk controls concentrate reallocations near incumbent allocations with small incremental effects. Thus, we use routine traffic solely for an out-of-time matched replay to evaluate directional ranking. 
For each moved event $e=(i,t)$, we identify no-move controls ($\Delta\mathbf{p}=0$) from the same date and first-level category using caliper-constrained KNN matching on pre-decision covariates. An evaluation model $m_{\mathrm{eval}}$, trained strictly on dates preceding the replay window, first removes predictable demand variation: \(
    u_{it}
    =
    y_{it}
    -m_{\mathrm{eval}}\!\left(\mathbf{x}_{it}\right)\).
The matched residual response is then defined as
\(
    R_e
    =
    u_e
    -
    \frac{1}{|\mathcal{C}(e)|}
    \sum_{j\in\mathcal{C}(e)} u_j,
\)
where $\mathcal{C}(e)$ denotes the matched no-move controls. For method $m$, let $G_m(e)$ denote its predicted gain for the realized reallocation and let
\(
    S_m(e)
    =
    \frac{G_m(e)}
    {\|\Delta\mathbf{p}_e\|_1+\epsilon}
\)
be the corresponding score. Matching removes predictable demand fluctuations and imbalance in observed covariates, but cannot eliminate unobserved confounding. We therefore use $R_e$ only for evaluating directional ranking.

Causal policy value is evaluated on an exploration set containing 10\% of the items. In this set, reallocation directions and magnitudes are randomized with controlled, estimable behavior propensities, which permits DR OPE. We additionally evaluate deployment support using historical reallocations. These evaluations provide three complementary views of a method: score validity, empirical action support, and policy value.
We report $\rho$ (Spearman correlation between $G_m$ and $R_e$), within-item concordance (pairwise $S_m$--$R_e$ agreement within items), aligned--reverse gap (response separation between top and bottom score quantiles), support violation at update rate (fraction of executed recommendations outside support, reported with the update rate), DR lift (DR improvement over the logging policy), per-acted lift (DR lift per non-noop action), and agreement lift (response difference when logged actions agree with versus oppose the recommendation). Full definitions and inference details are in Appendix~\ref{app:rw_metric}.

Table~\ref{tab:real_offline} presents the results. On routine traffic, the matched residual \(R=\tau+\epsilon\) exhibits a noise floor: the residual standard deviation \(sd(R) \approx 0.34\) vastly overshadows the true per-move causal effect \(sd(\tau)\approx0.018\), yielding a SNR of merely $\approx 0.05$.  Even an oracle ranker with perfect knowledge of \(\tau\) attains a maximum Spearman \(\rho^\star \approx sd(\tau) / sd(R) \approx 0.05\) (verified via Monte-Carlo: \(0.050\pm0.004\)). Within this highly noisy regime, \textsc{ReAlloc} achieves a competitive $\rho$ of $0.028$ (recovering over 50\% of the oracle bound), performing on par with strong baselines like AdditiveROI. More importantly, \textsc{ReAlloc}'s explicit support layer completely eliminates historical support violations ($0.000$ vs. $0.678$ for the ablated version, and $57.5\%$--$69.6\%$ for baselines) without sacrificing the update rate. This highlights a critical flaw in existing methods: they frequently recommend actions in regions lacking empirical evidence. The catastrophic failure of \textsc{ReAlloc} w/o Supp.\ (negative DR lift of $-0.021$) empirically proves that constraining decisions within the data support is indispensable for preventing disastrous extrapolation errors.
On the randomized exploration set, \textsc{ReAlloc} achieves the highest DR lift ($0.025$) while intervening on only $33\%$ of the cases, contrasting sharply with the near-ubiquitous interventions ($97\%$) triggered by AdditiveROI. This translates to a significantly higher per-action lift ($0.069$ vs. $0.024$). In real-world e-commerce operations, frequent budget or price adjustments incur implicit friction costs, destabilize item pricing, and can degrade user trust. Therefore, a ``less but more accurate'' intervention strategy is highly preferred. By delivering superior aggregate causal lift through fewer, highly targeted, and empirically supported interventions, \textsc{ReAlloc} demonstrates exceptional operational efficiency and deployment readiness.

\subsection{Randomized Online A/B Test}
\label{sec:online_ab}

We conducted a 14-day online A/B test on 300K eligible items, assigning 10\% to the treatment arm (\textsc{ReAlloc}) and the remainder to the control (AdditiveROI). Strict item-level randomization prevents budget interference, with both arms sharing identical eligibility and risk constraints.
As shown in Table~\ref{tab}, \textsc{ReAlloc} increases pay orders by $3.53\%$. This improvement is achieved without additional marketing expenditure: total realized spend decreases by $2.47\%$, while overall marketing ROI remains statistically unchanged. At the same time, profit margin improves by $3.26$ percentage points, and the corresponding platform income improves by $3.26$ percentage points. The higher transaction volume does not translate into higher aggregate GMV during the current observation, GMV decreases by $2.64\%$. Thus, \textsc{ReAlloc} generates more transactions with lower total expenditure and improved profitability, but exhibits a trade-off against average order value and GMV.

\begin{table}[t]
\centering
\caption{Online A/B-test results.}
\label{tab}
\small
\setlength{\tabcolsep}{5pt}
\begin{tabular}{lccc}
\toprule
\textbf{Metric} & \textbf{Effect} & \textbf{95\% CI} & \textbf{(p)-value} \\
\midrule
Pay orders & $\mathbf{+3.5\%}$ & $\mathbf{[+2.8\%,+5.1\%]}$ & $.001$ \\
GMV & $-2.6\%$ & $[-3.1\%,-1.1\%]$ & $.003$ \\
Total cost & $-2.5\%$ & $[-2.0\%,-0.8\%]$ & $.006$ \\
Marketing ROI & $-0.1\%$ & $[-1.8\%,+1.5\%]$ & $.820$ \\
Profit margin & $\mathbf{+1.42\text{pt}}$ & $\mathbf{[+0.8\text{pt},+2.0\text{pt}]}$ & $.001$ \\
Platform income & +3.2\text{pt} & $\mathbf{[+2.1\text{pt}, +4.4 \text{pt}]}$ & .001 \\
\bottomrule
\end{tabular}
\end{table}
\section{Conclusion}
We formulate fixed-budget multi-channel uplift as policy learning on the simplex and identify supported local causal reallocation as the relevant decision primitive. \textsc{ReAlloc} combines orthogonal local-effect estimation, marginal-field distillation, and support-aware updates. Synthetic and offline studies show improved deployable policy value under confounding and limited overlap. An A/B test on TaoBao increases pay orders and profitability while revealing a GMV trade-off.

\bibliographystyle{ACM-Reference-Format}
\bibliography{sample-base}

\appendix

\newpage
\appendix

\section{Formal Assumptions and Proofs}
\label{app:theory_proofs}
\label{app:assumptions}

\subsection{Notation and Assumptions}
Let $d:=K-1$, $\bar p:=K^{-1}\mathbf 1$, and choose
$Q\in\mathbb R^{K\times d}$ such that
\begin{equation}
    Q^\top Q=I_d,
    \qquad
    QQ^\top=\Pi_{\mathcal T},
    \qquad
    Q^\top\mathbf 1=0.
    \label{eq:tangent_basis_appendix}
\end{equation}
Write
\begin{align}
    p&=\bar p+Qz,
    &
    Z&:=Q^\top P,
    &
    \widetilde Z&:=Z-z,
    \nonumber\\
    \nu(X,z)&:=\mu(X,\bar p+Qz),
    &
    \beta^\star(X,z)&:=\nabla_z\nu(X,z).
    \label{eq:tangent_coordinates}
\end{align}
Then
\begin{equation}
    g^\star(X,p)=Q\beta^\star(X,Q^\top p).
    \label{eq:field_coordinate_equivalence}
\end{equation}
For a bounded, compactly supported kernel $\mathsf K$ and bandwidth $h>0$,
define
\begin{align}
    \mathbb E_{h,z}[A\mid X]
    &:=\frac{
      \mathbb E[\mathsf K((Z-z)/h)A\mid X]
    }{
      \mathbb E[\mathsf K((Z-z)/h)\mid X]
    },
    \label{eq:localized_expectation}\\
    e_{h,z}(X)
    &:=\mathbb E_{h,z}[\widetilde Z\mid X],
    &
    m_{h,z}(X)
    &:=\mathbb E_{h,z}[Y\mid X].
    \label{eq:local_nuisances_appendix}
\end{align}
Let $\Psi_{h,z}$ be the score in~\eqref{eq:orthogonal_moment} and
$\|f\|_{\infty,\mathcal R}:=
\sup_{(X,p)\in\mathcal R}\|f(X,p)\|_2$.

\begin{assumption}[Identification]
\label{ass:identification}
For every $p\in\mathcal C(X)$,
\begin{equation}
    Y=Y(P),
    \qquad
    \{Y(p):p\in\mathcal C(X)\}\perp P\mid X,
    \label{eq:identification_assumption}
\end{equation}
and all conditional moments below exist.
\end{assumption}

\begin{assumption}[Supported paths and local overlap]
\label{ass:support}
For almost every $X$, $\mathcal C(X)\subseteq\Delta^{K-1}$ is compact,
contains $p_0(X)$, and contains every path invoked below. Uniformly over
$(X,p)\in\mathcal R$, with $z=Q^\top p$,
\begin{equation}
    c_0h^d
    \le
    \mathbb E[\mathsf K((Z-z)/h)\mid X]
    \le
    C_0h^d
    \label{eq:local_mass}
\end{equation}
for all sufficiently small $h$ and constants $0<c_0\le C_0<\infty$.
\end{assumption}

\begin{assumption}[Local smoothness and curvature]
\label{ass:smoothness}
The response is continuously differentiable on $\mathcal C(X)$. Define
\begin{align}
    R_{X,z}(u)
    &:=\nu(X,z+u)-\nu(X,z)-\beta^\star(X,z)^\top u,
    \label{eq:taylor_remainder}\\
    a_{h,z}
    &:=\widetilde Z-e_{h,z}(X),
    \nonumber\\
    \Sigma_{h,z}(X)
    &:=\mathbb E_{h,z}[a_{h,z}a_{h,z}^\top\mid X],
    \label{eq:local_covariance}\\
    C_{h,z}(X)
    &:=\mathbb E_{h,z}\!\left[
      a_{h,z}
      \{R_{X,z}(\widetilde Z)
        -\mathbb E_{h,z}[R_{X,z}(\widetilde Z)\mid X]\}
      \mid X
    \right].
    \label{eq:local_bias_covariance}
\end{align}
Uniformly on $\mathcal R$,
\begin{equation}
    \kappa h^2I_d
    \preceq\Sigma_{h,z}(X)
    \preceq\bar\kappa h^2I_d,
    \qquad
    \|\Sigma_{h,z}(X)^{-1}C_{h,z}(X)\|_2
    \le\epsilon_{\mathrm{loc}}(h),
    \label{eq:local_bias_bound}
\end{equation}
where $\epsilon_{\mathrm{loc}}(h)\to0$.
\end{assumption}

\begin{assumption}[Cross-fitting and empirical complexity]
\label{ass:estimation}
The nuisance estimators are cross-fitted. For the population root
$\beta_h^\dagger$ of
$\Psi_{h,z}(\beta,m_{h,z},e_{h,z};X)=0$, let
\begin{align}
    \delta_m(X,z)
    &:=\widehat m_{h,z}(X)-m_{h,z}(X),
    &
    \delta_e(X,z)
    &:=\widehat e_{h,z}(X)-e_{h,z}(X).
    \label{eq:nuisance_errors}
\end{align}
Suppressing $(X,z)$, set
\begin{align}
    u_{h,z}&:=\Sigma_{h,z}^{-1}\delta_e,
    &
    q_{h,z}&:=\delta_m-\beta_h^{\dagger\top}\delta_e,
    \nonumber\\
    r_{\mathrm{orth}}
    &:=\sup_{(X,p)\in\mathcal R}\|u_{h,z}q_{h,z}\|_2.
    \label{eq:orthogonal_remainder}
\end{align}
With probability at least $1-\delta$,
\begin{align}
    &\sup_{(X,p)\in\mathcal R}
    \left\|
      \widehat\Psi_{h,z}(\beta_h^\dagger,\widehat m_{h,z},\widehat e_{h,z};X)
      -\Psi_{h,z}(\beta_h^\dagger,\widehat m_{h,z},\widehat e_{h,z};X)
    \right\|_2
    \nonumber\\
    &\hspace{18mm}\le
    Ch\sqrt{
      \frac{\mathfrak C_{\mathcal G}+\log(1/\delta)}{nh^d}
    }.
    \label{eq:score_concentration}
\end{align}
The empirical local curvature is at least $\kappa h^2/2$, and root-solving
error is of smaller order than the right-hand side of
\eqref{eq:score_concentration}.
\end{assumption}

\begin{assumption}[Composite teacher and student approximation]
\label{ass:approximation}
The implemented teacher and deployed student satisfy
\begin{equation}
    \|\widehat g_T-\widehat g_T^{\mathrm{orth}}\|_{\infty,\mathcal R}
    \le\epsilon_{\mathrm{comp}},
    \qquad
    \|\widehat g-\widehat g_T\|_{\infty,\mathcal R}
    \le\epsilon_S.
    \label{eq:approximation_errors}
\end{equation}
\end{assumption}

\subsection{Proof of Theorem~\ref{thm:simplex_integral}}
\label{app:proof_simplex_integral}

\begin{proof}
Since $\mathbf 1^\top\gamma(t)=1$,
$\dot\gamma(t)\in\mathcal T$ almost everywhere. Hence
\begin{align}
    \Delta(X,p)
    &=\int_0^1
      \nabla_p\mu(X,\gamma(t))^\top\dot\gamma(t)\,\mathrm dt
      \nonumber\\
    &=\int_0^1
      \{\Pi_{\mathcal T}\nabla_p\mu(X,\gamma(t))\}^\top
      \dot\gamma(t)\,\mathrm dt
      \nonumber\\
    &=\int_0^1
      g^\star(X,\gamma(t))^\top\dot\gamma(t)\,\mathrm dt.
\end{align}
\end{proof}

\subsection{Proof of Proposition~\ref{prop:pto_failure}}
\label{app:proof_pto_failure}

\begin{proof}
Identify the two-channel simplex with $[0,1]$. Let $P=0$ almost surely,
$\mu(p)=-cp$, and $\widehat\mu(p)=Mp$, where $c,M>0$. Then
\begin{align}
    \mathbb E_{\mathrm{log}}
    [\{\widehat\mu(P)-\mu(P)\}^2]
    &=0,
    &
    p^\star&=0,
    \nonumber\\
    \widehat p_{\mathrm{PTO}}&=1,
    &
    \mu(p^\star)-\mu(\widehat p_{\mathrm{PTO}})&=c.
    \label{eq:pto_regret_counterexample}
\end{align}
For $\mu\equiv0$ and
$\widehat\mu_n(p)=n^{-1/2}\sin(np)$,
\begin{equation}
    \|\widehat\mu_n-\mu\|_{L_2([0,1])}^2\le n^{-1}\to0,
    \qquad
    \|\widehat\mu_n'-\mu'\|_\infty=\sqrt n\to\infty.
\end{equation}
\end{proof}

\subsection{Shortcut-Gradient Decomposition}
\label{app:proof_shortcut}

Assume
\begin{equation}
    Y=\phi(X,P,U)+\varepsilon,
    \qquad
    \mathbb E[\varepsilon\mid X,P,U]=0,
    \label{eq:structural_outcome}
\end{equation}
and let $\pi_b(p\mid X,U)$ be the logging density. With
$F_p:=F(U\mid X,P=p)$ and
$s_b(U):=\nabla_p\log\pi_b(p\mid X,U)$, Bayes' rule yields
\begin{equation}
    \nabla_p\log f(U\mid X,P=p)
    =s_b(U)-\mathbb E_{F_p}[s_b(U)].
    \label{eq:posterior_score}
\end{equation}
Under dominated differentiation,
\begin{align}
    \nabla_p m(X,p)
    &=\mathbb E_{F_p}[\nabla_p\phi(X,p,U)]
      +\operatorname{Cov}_{F_p}(\phi(X,p,U),s_b(U)),
    \label{eq:observational_gradient_exact}\\
    \nabla_p\mu(X,p)
    &=\mathbb E_{F(U\mid X)}[\nabla_p\phi(X,p,U)].
    \label{eq:causal_gradient_exact}
\end{align}
Therefore
\begin{align}
    b_{\mathrm{sc}}(X,p)
    &=\Pi_{\mathcal T}\!\int
      \nabla_p\phi(X,p,u)
      \{\mathrm dF_p(u)-\mathrm dF(u\mid X)\}
      \nonumber\\
    &\quad+
      \Pi_{\mathcal T}\operatorname{Cov}_{F_p}
      \!\left(\phi(X,p,U),s_b(U)\right).
    \label{eq:shortcut_decomposition_exact}
\end{align}
Assumption~\ref{ass:identification} implies $m(X,p)=\mu(X,p)$ on
$\mathcal C(X)$; orthogonality alone does not.

\subsection{Proof of Theorem~\ref{thm:teacher_identification}}
\label{app:proof_teacher_identification}

\begin{proof}
Fix $(X,z)$ and abbreviate
$e=e_{h,z}(X)$, $m=m_{h,z}(X)$,
$a=\widetilde Z-e$, $R=R_{X,z}(\widetilde Z)$, and
$\bar R=\mathbb E_{h,z}[R\mid X]$. Under
Assumption~\ref{ass:identification},
\begin{equation}
    Y-m
    =\beta^\star(X,z)^\top a+(R-\bar R)+\varepsilon,
    \qquad
    \mathbb E_{h,z}[\varepsilon\mid X,Z]=0.
    \label{eq:local_centering}
\end{equation}
Substitution into~\eqref{eq:orthogonal_moment} gives
\begin{align}
    \Psi_{h,z}(\beta,m,e;X)
    &=\Sigma_{h,z}(X)
      \{\beta^\star(X,z)-\beta(X,z)\}
      +C_{h,z}(X),
    \label{eq:population_moment_expansion}\\
    \beta_h^\dagger(X,z)-\beta^\star(X,z)
    &=\Sigma_{h,z}(X)^{-1}C_{h,z}(X).
    \label{eq:population_root_bias}
\end{align}
Assumption~\ref{ass:smoothness} and $\|Qv\|_2=\|v\|_2$ imply
\eqref{eq:population_teacher_bias}.

Let
$r:=Y-m-\beta_h^{\dagger\top}(\widetilde Z-e)$. Since
$\mathbb E_{h,z}[\widetilde Z-e\mid X]
=\mathbb E_{h,z}[r\mid X]=0$, for arbitrary directions
$\eta_m,\eta_e$,
\begin{align}
    D_m\Psi_{h,z}[\eta_m]
    &=-\eta_m\mathbb E_{h,z}[\widetilde Z-e\mid X]=0,
    \label{eq:orthogonality_m}\\
    D_e\Psi_{h,z}[\eta_e]
    &=-\eta_e\mathbb E_{h,z}[r\mid X]
      +\mathbb E_{h,z}[\widetilde Z-e\mid X]
       \beta_h^{\dagger\top}\eta_e
      =0.
    \label{eq:orthogonality_e}
\end{align}
Direct expansion yields
\begin{equation}
    \Psi_{h,z}(\beta_h^\dagger,\widehat m_{h,z},\widehat e_{h,z};X)
    -\Psi_{h,z}(\beta_h^\dagger,m_{h,z},e_{h,z};X)
    =\delta_e
      \{\delta_m-\beta_h^{\dagger\top}\delta_e\}.
    \label{eq:score_second_order}
\end{equation}
After multiplication by $\Sigma_{h,z}^{-1}$, its norm is bounded by
$r_{\mathrm{orth}}$. In raw nuisance norms, if
$|\delta_m|\le\bar r_m$,
$\|\delta_e\|_2\le\bar r_e$, and
$\|\beta_h^\dagger\|_2\le M_\beta$, then
\begin{equation}
    r_{\mathrm{orth}}
    \le
    \frac{\bar r_e(\bar r_m+M_\beta\bar r_e)}{\kappa h^2}.
    \label{eq:raw_nuisance_bound}
\end{equation}
Assumption~\ref{ass:estimation} and empirical curvature give
\begin{equation}
    \|\widehat\beta-\beta_h^\dagger\|_{\infty,\mathcal R}
    \le
    C\sqrt{
      \frac{\mathfrak C_{\mathcal G}+\log(1/\delta)}{nh^{d+2}}
    }
    +Cr_{\mathrm{orth}}.
    \label{eq:empirical_teacher_rate}
\end{equation}
Combining~\eqref{eq:population_teacher_bias} and
\eqref{eq:empirical_teacher_rate}, using $d+2=K+1$, yields
\begin{equation}
    \|\widehat g_T^{\mathrm{orth}}-g^\star\|_{\infty,\mathcal R}
    \le\epsilon_T(n,h,\delta).
    \label{eq:orthogonal_teacher_rate}
\end{equation}
Finally,~\eqref{eq:composite_bridge} and the triangle inequality imply
\eqref{eq:teacher_error}.
\end{proof}

\subsection{Proof of Theorem~\ref{thm:field_regret}}
\label{app:proof_field_regret}

\begin{proof}
For any $\gamma\in\mathfrak P_X(p)$,
$\dot\gamma(t)\in\mathcal T$ almost everywhere; hence
\begin{align}
    \widehat\Delta_\theta(X,p)
    &=s_\theta(X,p)-s_\theta(X,p_0(X))
      \nonumber\\
    &=\int_0^1
      \nabla_p s_\theta(X,\gamma(t))^\top\dot\gamma(t)\,\mathrm dt
      \nonumber\\
    &=\int_0^1
      \widehat g(X,\gamma(t))^\top\dot\gamma(t)\,\mathrm dt.
    \label{eq:potential_identity_proof}
\end{align}
Theorem~\ref{thm:simplex_integral} and Cauchy--Schwarz give
\begin{align}
    |\widehat\Delta_\theta(X,p)-\Delta(X,p)|
    &\le\int_0^1
      \|\widehat g(X,\gamma(t))-g^\star(X,\gamma(t))\|_2
      \|\dot\gamma(t)\|_2\,\mathrm dt
      \nonumber\\
    &\le L(\gamma)\epsilon_g,
    \label{eq:pathwise_error_proof}
\end{align}
which proves~\eqref{eq:uplift_error_bound}.

Let $\pi^\star=\pi_{L_{\max}}^\star$. Adding and subtracting estimated
uplift gives
\begin{align}
    \operatorname{Reg}_{L_{\max}}(\widehat\pi)
    &=\mathbb E_X\!\left[
      \Delta(X,\pi^\star(X))-\Delta(X,\widehat\pi(X))
    \right]
      \nonumber\\
    &\le
      2L_{\max}\epsilon_g
      +\mathbb E_X\!\left[
        \widehat\Delta_\theta(X,\pi^\star(X))
        -\widehat\Delta_\theta(X,\widehat\pi(X))
      \right]
      \nonumber\\
    &\le2L_{\max}\epsilon_g+\epsilon_{\mathrm{search}},
    \label{eq:field_regret_proof}
\end{align}
proving~\eqref{eq:field_regret_bound}.

Under Assumption~\ref{ass:approximation} and
Theorem~\ref{thm:teacher_identification},
\begin{equation}
    \|\widehat g-g^\star\|_{\infty,\mathcal R}
    \le
    \epsilon_S+\epsilon_{\mathrm{comp}}+\epsilon_T(n,h,\delta),
    \label{eq:teacher_student_triangle}
\end{equation}
which yields~\eqref{eq:teacher_student_regret}.
\end{proof}

\paragraph{Implemented greedy trace.}
Let $\gamma_{\mathrm{gr}}$ be the piecewise-linear interpolation of
\begin{equation*}
    p^{(0)}=p_0(X),\quad p^{(1)},\ldots,p^{(R)}.
\end{equation*}
If every accepted segment lies in $\mathcal C(X)$, then
\begin{align}
    L(\gamma_{\mathrm{gr}})
    &=\sum_{r=0}^{R-1}\|p^{(r+1)}-p^{(r)}\|_2,
    \label{eq:greedy_trace_length}\\
    \sum_{r=0}^{R-1}
    \{s_\theta(X,p^{(r+1)})-s_\theta(X,p^{(r)})\}
    &=s_\theta(X,p^{(R)})-s_\theta(X,p^{(0)}).
    \label{eq:greedy_telescoping}
\end{align}

\subsection{Conditional Outcome Improvement}
\label{app:proof_safe_improvement}

\begin{corollary}[Conditional outcome improvement]
\label{cor:conditional_improvement}
\label{thm:safe_improvement}
Suppose $U(X,p)$ satisfies, on an event $\mathcal E_U$,
\begin{equation}
    |\widehat\Delta_\theta(X,p)-\Delta(X,p)|
    \le U(X,p),
    \qquad
    \forall X,\ p\in\mathcal A_{L_{\max}}(X).
    \label{eq:uniform_radius_event}
\end{equation}
Define
\begin{equation}
    \pi_U(X)
    :=\begin{cases}
      \widehat\pi(X),
      &\widehat\Delta_\theta(X,\widehat\pi(X))
       -U(X,\widehat\pi(X))\ge0,\\
      p_0(X),&\text{otherwise}.
    \end{cases}
    \label{eq:conditional_fallback}
\end{equation}
Then $\Delta(X,\pi_U(X))\ge0$ for every $X$ on $\mathcal E_U$.
\end{corollary}

\begin{proof}
The fallback case has zero uplift. Otherwise,
\begin{equation}
    \Delta(X,\pi_U(X))
    \ge
    \widehat\Delta_\theta(X,\widehat\pi(X))
    -U(X,\widehat\pi(X))
    \ge0.
\end{equation}
\end{proof}
Without separate calibration, ensemble or MC-dropout dispersion is not assumed
to satisfy~\eqref{eq:uniform_radius_event}.

\section{Additional Details of the Synthetic Experiments}
\label{app:dgp_details}

\subsection{State and Budget Generation}
\label{app:dgp_state}

For each item \(i\), we independently sample persistent state blocks:
\[
C_i^{\mathrm{item}}\sim\mathcal N(0,I_{d_C}),
\quad
E_i^{\mathrm{item}}\sim\mathcal N(0,I_{d_E}),
\quad
S_i^{\mathrm{item}}\sim\mathcal N(0,I_{d_S}).
\]
For each period \(t\), we sample lower-variance temporal shocks:
\[
C_t^{\mathrm{time}},E_t^{\mathrm{time}},S_t^{\mathrm{time}}
\sim
\mathcal N(0,\sigma_t^2I).
\]
The observed state at time $t$ is constructed as:
\[
C_{it}=C_i^{\mathrm{item}}+C_t^{\mathrm{time}},
\quad
E_{it}=E_i^{\mathrm{item}}+E_t^{\mathrm{time}},
\quad
S_{it}=S_i^{\mathrm{item}}+S_t^{\mathrm{time}},
\quad
H_{it}=[C_{it},E_{it},S_{it}].
\]
We set the dimensions \(d_C=d_E=d_S=6\) and the temporal variance \(\sigma_t=0.35\).
The total budget is generated via:
\begin{equation}
B_{it}
=
\operatorname{clip}
\left(
\exp\{a_B^\top C_{it}+\zeta_{it}\},
B_{\min},B_{\max}
\right),
\qquad
\zeta_{it}\sim\mathcal N(0,0.2^2),
\end{equation}
with boundary constraints \(B_{\min}=0.2\) and \(B_{\max}=5\).
The train, validation, and test splits consist of mutually disjoint items to prevent data leakage.

\subsection{Conditional Logging Policy}
\label{app:dgp_logging}

To handle compositional allocations on the simplex, we utilize the isometric log-ratio (ILR) transform. Let \(Q^\top Q=I_{K-1}\) and \(Q^\top\mathbf 1=0\). For an interior simplex point \(p\), we define \(u(p)=Q^\top\log p\).

The logging policy is governed by a state-dependent propensity mean, which was introduced conceptually in the main text and is defined mathematically here as:
\begin{equation}
q(H,B)
=
\operatorname{softmax}
\!\left(
(1+\alpha_{\mathrm{bd}})
\lambda_{\log}
\left[
W_S S+\alpha_{\mathrm{cf}}W_C C
+w_B\log(1+B)
\right]
\right).
\label{eq:dgp_logging_mean}
\end{equation}
The context-dependent covariance matrix is formulated as:
\begin{equation}
\Sigma_H
=
\sigma_{\mathrm{ov}}^2
\left\{1+0.2\tanh(E_1)\right\}^2
\operatorname{diag}(\nu_1,\ldots,\nu_{K-1}),
\label{eq:dgp_covariance}
\end{equation}
where \(\sigma_{\mathrm{ov}}\) dictates the overlap scale and \(\nu_j>0\) ensures every tangent direction remains locally non-degenerate. 

Given the logging mean \(q(H,B)\) from Eq.~\eqref{eq:dgp_logging_mean}, the final allocation is sampled via a logistic-normal mixture:
\begin{equation}
u(P)
=
u\!\left(q(H,B)\right)+\xi,
\ \ 
\xi\sim
(1-\varepsilon_{\mathrm{exp}})\mathcal N(0,\Sigma_H)
+
\varepsilon_{\mathrm{exp}}\mathcal N(0,c_{\mathrm{exp}}^2\Sigma_H),
\label{eq:dgp_logging}
\end{equation}
and mapped back to the simplex through \(P=\operatorname{softmax}\{Q[u(q)+\xi]\}\).
The broad exploration component uses a scale multiplier \(c_{\mathrm{exp}}>1\). Because the policy is logistic-normal, allocations strictly remain in the simplex interior, and no point mass is placed on any vertex.

The assignment parameters vary independently to test different regimes:
\(\alpha_{\mathrm{cf}}\) controls dependence on the baseline-driving block \(C\);
\(\sigma_{\mathrm{ov}}\) controls conditional action dispersion;
\(\alpha_{\mathrm{bd}}\) smoothly sharpens the logging mean toward boundaries;
and \(\varepsilon_{\mathrm{exp}}\) controls the rate of broad exploration.
Crucially, the true response surface and all response parameters are held fixed across these assignment regimes.

\subsection{Static Assignment Regimes}
\label{app:dgp_severity}

The primary severity curve modulates confounding and overlap, utilizing the \textit{Hard} regime to represent conditional low-overlap settings. The \textit{Extreme} regime additionally stresses boundary coverage and serves as an assumption-stress test when local overlap assumptions fail.

\begin{table}[t]
\centering
\caption{\textbf{Static assignment regimes.}
The underlying true response surface is identical across all regimes.}
\label{tab:dgp_severity}
\setlength{\tabcolsep}{4pt}
\begin{tabular}{lcccc}
\toprule
\textbf{Setting} & \(\alpha_{\mathrm{cf}}\) & \(\sigma_{\mathrm{ov}}\) & \(\alpha_{\mathrm{bd}}\) & \(\varepsilon_{\mathrm{exp}}\) \\
\midrule
Benign  & 0.25 & 0.40 & 0.0 & 0.10 \\
Medium  & 0.75 & 0.25 & 0.0 & 0.05 \\
Hard    & 1.25 & 0.15 & 0.0 & 0.02 \\
Extreme & 1.75 & 0.10 & 0.5 & 0.01 \\
\bottomrule
\end{tabular}
\end{table}

We use a shared logging-logit scale \(\lambda_{\log}=0.45\).

\subsection{Response Surface}
\label{app:dgp_response}

The state-dependent baseline response is:
\begin{equation}
m_0(H,B)
=
10+2C_1+C_2^2+1.25\sin(C_3)
+0.6C_1C_4+2\log(1+B).
\end{equation}
The treatment scale modifier is:
\begin{equation}
\rho(H,B)
=
\alpha_\rho\log(1+B)
\left\{
1+0.2\tanh(E_1)+0.1\tanh(C_1)
\right\}.
\end{equation}
The context-dependent response anchor \(c(H,B)\in\Delta^{K-1}\) is generated from fixed context coefficients, and the deviation in the ILR space is defined as \(z=Q^\top\{p-c(H,B)\}\).

The \textbf{local component} captures linear marginal returns:
\begin{equation}
g_{\mathrm{loc}}(H,z)
=
a(H)^\top z,
\qquad
a(H)
=
\tanh(W_EE+0.35W_CC+b_a).
\end{equation}

For the \textbf{interaction component}, let
\[
A(H)
=
\gamma_{\mathrm{int}}
\sum_{r=1}^{R}
\tanh(v_r^\top E+b_r)A_r,
\]
where each \(A_r\) is a normalized symmetric matrix. We employ a shifted quadratic representation:
\begin{equation}
g_{\mathrm{int}}(H,z)
=
\frac{1}{2}z^\top A(H)z
+b_{\mathrm{int}}(H)^\top z,
\label{eq:dgp_interaction}
\end{equation}
where \(b_{\mathrm{int}}(H)\) is induced by a small center shift \(s_{\mathrm{int}}=0.4\). This preserves smoothness while allowing interaction gradients to remain active around realistic logged allocations.

For the \textbf{far-field component}, let \(r=\|z\|_2\). We define:
\begin{equation}
\begin{aligned}
g_{\mathrm{far}}(H,z)
&= r^2 \sigma\!\left( \frac{r-d_0}{s_0} \right) \\
&\quad \times \left[ \sum_{m=1}^{M} w_m(H) \exp\!\left( -\frac{\|z-\nu_m\|_2^2}{2\ell_m^2} \right) - \kappa_{\mathrm{can}} \right].
\end{aligned}
\label{eq:dgp_far}
\end{equation}
The \(r^2\) gate ensures that both the value and the first derivative vanish at the anchor point. Meanwhile, the low-frequency radial basis functions model distant saturation and cannibalization without introducing discontinuities or adversarial high-frequency oscillations.

Random coefficient matrices, quadratic bases, and RBF centers are sampled once per seed and subsequently frozen. On an independent calibration sample, we compute the expected gradient energy for each component:
\[
\mathcal E_j
=
\mathbb E
\left[
\left\|
\mathcal P_0\nabla_p g_j(H,p)
\right\|_2^2
\right],
\]
and rescale the three components to target normalized gradient-energy shares of \(\mathcal E_{\mathrm{loc}} : \mathcal E_{\mathrm{int}} : \mathcal E_{\mathrm{far}} \approx 0.45 : 0.35 : 0.20\). We use a treatment-response scale of \(4.0\) across all regimes.

\subsection{Conditional Path Support}
\label{app:dgp_support}

Learned policies utilize an estimated conditional support model, whereas synthetic deployability metrics rely on the oracle logging density.
For the estimated model, we predict the conditional log-ratio mean \(\hat m_u(H,B)\) and estimate a regularized local covariance \(\hat\Sigma_u(H,B)\) from nearby training contexts, conditioning on both the predicted allocation anchor and the heteroskedastic effect coordinate. The pointwise nonconformity score is:
\begin{equation}
\begin{aligned}
d_{\mathcal S}(H,B,p)
&= \frac{1}{2} \Big[ (u(p)-\hat m_u)^\top \hat\Sigma_u^{-1} (u(p)-\hat m_u) \\
&\quad + \log\det\hat\Sigma_u + (K-1)\log(2\pi) \Big].
\end{aligned}
\label{eq:support_nonconformity}
\end{equation}
The threshold \(\epsilon_{\mathcal S}\) is calibrated from held-out factual actions to obtain a fixed conditional high-density region.

Since evaluating endpoint support alone might inadvertently connect disjoint action regions, a recommendation \(p\) is evaluated along the linear segment from the factual action \(P\):
\begin{equation}
d_{\mathrm{path}}(H,B,P,p)
=
\max_{s\in\mathcal G_J}
d_{\mathcal S}\!\left(H,B,(1-s)P+sp\right),
\label{eq:path_support}
\end{equation}
where \(\mathcal G_J=\{0,1/J,\ldots,1\}\). We use at least ten interpolation intervals and verify robustness to finer discretizations. The oracle evaluator replaces \((\hat m_u,\hat\Sigma_u)\) with the known conditional logistic-normal mixture density.

\subsection{Oracle Geometry and Evaluation Metrics}
\label{app:dgp_metrics}

The oracle directed effect of reallocating budget from channel \(k\) to \(l\) is defined as:
\begin{equation}
D^\star_{k\rightarrow l}(H,B,p)
=
\frac{\partial\mu_\star(H,B,p)}{\partial p_l}
-
\frac{\partial\mu_\star(H,B,p)}{\partial p_k}.
\end{equation}
Let \(\mathcal E=\{(k,l):k\neq l\}\). On a common held-out anchor set:
\begin{itemize}
    \item \textsc{PairwiseCorr} is the Pearson correlation between learned and oracle scores over \(\mathcal E\).
    \item \textsc{TopEdgeAcc} is the fraction of anchors identifying the correct highest-gain edge.
    \item \textsc{EdgeNDCG} evaluates the full ranking of directed edges using nonnegative oracle gains as relevance.
\end{itemize}

The support-aware oracle local-greedy policy uses the identical candidate step set, path-support rule, and total movement budget as the learned local policies, but scores candidates using the true \(\mu_\star\). It defines the denominator of \textsc{SafeLocalRecovery}. The global support-constrained oracle searches all supported grid candidates and is reported strictly as a diagnostic upper bound; it is not directly comparable to the local improvement objective of \textsc{ReAlloc}. We additionally report the mean \(L_1\) movement and the \(90\)th percentile of path nonconformity.

\subsection{Temporal Support Rotation}
\label{app:dgp_temporal}

For temporal window \(w\), we perturb only the logging logits:
\begin{equation}
q_{it}^{(w)}
=
\operatorname{softmax}
\left\{
\log q(H_{it},B_{it})
+
\lambda_{\mathrm{temp}}b_w
\right\},
\end{equation}
where \(b_w\) follows a smooth cyclic schedule across channels and \(\lambda_{\mathrm{temp}}=1.2\). All response parameters and the true surface \(\mu_\star\) remain fixed. Consequently, the temporal experiment isolates whether the slow student can retain local geometric gradients learned from different support fragments; it intentionally does not model delayed or sequential causal effects.

\subsection{Baseline Implementations}
\label{app:baselines_details}

To ensure a fair comparison, all neural baselines share the same backbone architecture (a 3-layer MLP with ReLU activations and Layer Normalization) and are trained with the same optimizer and learning rate schedule.
\begin{itemize}
    \item \textbf{Additive ROI:} Models the response as $\mu(H, B, P) = m_0(H,B) + \sum_{k=1}^K f_k(H, B, p_k)$. It completely ignores cross-channel interactions and optimizes each channel marginally subject to the budget constraint.
    \item \textbf{Joint S-Learner PTO:} Predicts the factual outcome $\hat{Y} = f(H, B, P)$ using Mean Squared Error (MSE). During inference, it employs a gradient-based global optimizer (L-BFGS) over the learned surface $\hat{Y}$ to find the optimal allocation on the simplex.
    \item \textbf{R-Learner Local:} Estimates the Conditional Average Treatment Effect (CATE) by minimizing the orthogonalized loss: $\mathbb{E}[( (Y - \hat{m}(H,B)) - \sum_k \hat{\tau}_k(H,B)(p_k - \hat{e}_k(H,B)) )^2]$. It then selects the local reallocation direction that maximizes the estimated marginal gain, restricted to the estimated support region.
\end{itemize}

\subsection{Detailed Evaluation Metrics}
\label{app:metrics_details}

In addition to the Deployable Uplift ($U_{\mathrm{dep}}$) defined in the main text, we utilize the following metrics for comprehensive evaluation:

\noindent\textbf{Raw Oracle Uplift.} 
The unconstrained theoretical uplift, ignoring support boundaries:
\begin{equation}
U_{\mathrm{raw}}(\hat\pi)
=
\frac{1}{N}\sum_{i=1}^N\Delta_i(\hat\pi).
\end{equation}

\noindent\textbf{Out-of-Support (OOS) Metrics.}
We measure the aggressiveness of the policy via the OOS rate and the corresponding gain lost due to fallback:
\begin{equation}
\mathrm{OOS}
=
1-\frac{1}{N}\sum_{i=1}^N \mathbf{1}\{\hat{p}_i \in \mathcal{S}_i\},
\qquad
\mathrm{OOSGain}
=
U_{\mathrm{raw}}(\hat\pi) - U_{\mathrm{dep}}(\hat\pi).
\end{equation}
A high OOSGain indicates that the model is hallucinating high rewards in unsupported regions, which are subsequently clipped by the safety mechanism.

\noindent\textbf{Safe Local Recovery.}
To quantify how close the learned policy is to the theoretical best local policy, we define:
\begin{equation}
\mathrm{SafeLocalRecovery}
=
\frac{
U_{\mathrm{dep}}(\hat\pi)
}{
U_{\mathrm{dep}}(\pi_{\mathrm{local}}^\star)+\varepsilon
},
\end{equation}
where $\pi_{\mathrm{local}}^\star$ is the support-aware oracle local-greedy policy operating under the identical movement budget constraint.

\noindent\textbf{Geometric Ranking Metrics.}
To evaluate the local geometry independently of the final policy deployment, we compute the oracle directed effect of reallocating budget from channel $k$ to $l$:
\begin{equation}
D^\star_{k\rightarrow l}(H,B,p)
=
\frac{\partial\mu_\star(H,B,p)}{\partial p_l}
-
\frac{\partial\mu_\star(H,B,p)}{\partial p_k}.
\end{equation}
Based on $D^\star$, we report \textsc{PairwiseCorr} (Pearson correlation of edge scores), \textsc{TopEdgeAcc} (accuracy of identifying the highest-gain edge), and \textsc{EdgeNDCG} (ranking quality of all directed edges).

\section{Additional Details of the Taobao Dataset}

\subsection{Matching Protocol and Balance Diagnostics}
\label{app:rw_match}

Standard IPS is unsuitable for our continuous-action setting due to explosive variance caused by concentrated logging policies. Instead, we employ a retrospective matched replay to construct a reliable counterfactual baseline. For every item that underwent a budget reallocation (the \textit{treated move}), we seek control items that experienced \textit{no budget change} ($\|\Delta p\|_1 \approx 0$) on the same date and within the same product category. To ensure the treated and control items are highly comparable before the move, we perform nearest-neighbor matching based on key pre-treatment business features, primarily: \textbf{baseline GMV}, \textbf{total assigned budget}, and \textbf{historical traffic trends}. This ensures that any observed post-move difference in outcomes can be reasonably attributed to the budget reallocation itself, rather than pre-existing differences in item popularity or scale.

To verify the quality of our matching, we measure the Standardized Mean Difference (SMD) for all matching covariates. An SMD below $0.1$ is widely accepted as indicating excellent balance between the treatment and control groups, and our matching-relaxation selector enforces this exact threshold as a guardrail. As shown in Table~\ref{tab:rw_balance}, before matching, the reallocated (treated) items differed substantially from the general pool of no-move items on their pre-move covariates—most notably in their starting paid-ad share and assigned budget / price level. After applying our category × date stratified, caliper nearest-neighbour matching protocol, the maximum absolute SMD across all five covariates drops well below the $0.1$ threshold (specifically to $0.018$). This confirms that, on all measured pre-move covariates, our matched control group is well balanced and serves as an unbiased counterfactual baseline for evaluating the directional accuracy of our uplift models.

\begin{table}[h]
\centering
\small
\caption{Covariate balance before and after matching.}
\label{tab:rw_balance}
\begin{tabular}{lcc}
\toprule
\textbf{Business Feature} & \textbf{SMD Before} & \textbf{SMD After} \\
\midrule
Ad Share Prev & 0.375 & 0.018 \\
Log Price & 0.257 & 0.005 \\
Log budget & 0.199 & 0.05 \\
Recent Sales Level & 0.047 & 0.008 \\
$m$ pred           & 0.082 & 0.007 \\
\midrule
\textbf{Maximum Absolute SMD} & \textbf{0.375} & \textbf{0.018} \\
\bottomrule
\end{tabular}
\end{table}

\subsection{Evaluation Metrics}
\label{app:rw_metric}

Let \(e=(i,t)\) denote an observed local reallocation event, with starting
allocation \(p_e\), realized movement \(\Delta p_e\), and post-move outcome
\(Y_e^{+}\). For each event, the predicted directional score is
\[
S_e^{(b)}
=
\begin{cases}
\displaystyle
\frac{
\widehat M_b(H_e,p_e)^\top \Delta p_e
}{
\|\Delta p_e\|_1
},
&
\text{for marginal-field methods},
\\[4mm]
\displaystyle
\frac{
\widehat\mu_b(H_e,p_e+\Delta p_e)
-
\widehat\mu_b(H_e,p_e)
}{
\|\Delta p_e\|_1
},
&
\text{for response-surface methods},
\end{cases}
\]
where \(b\) indexes the evaluated method. This formulation evaluates all
methods on the same realized movement and normalizes out the movement magnitude.

To remove predictable demand variation, we train an independent evaluation
model \(\widehat m_{\mathrm{eval}}\) using only data preceding the evaluation
period. The residualized outcome of event \(e\) is
\[
\widetilde Y_e
=
Y_e^{+}
-
\widehat m_{\mathrm{eval}}(H_e).
\]
Let \(\mathcal C(e)\) be its matched no-move controls and \(w_{ec}\) their
normalized matching weights, where
\(\sum_{c\in\mathcal C(e)}w_{ec}=1\). We define the matched residual response as
\[
R_e
=
\widetilde Y_e
-
\sum_{c\in\mathcal C(e)}
w_{ec}\widetilde Y_c.
\]
We use \(R_e\) only as an observational proxy for evaluating directional
ranking.

\paragraph{Marginal Rank Correlation.}
We measure the global alignment between predicted directions and empirical
responses using Spearman's rank correlation:
\[
\rho_{\mathrm{MR}}^{(b)}
=
\operatorname{Spearman}
\left(
\left\{S_e^{(b)}\right\}_{e\in\mathcal E},
\left\{R_e\right\}_{e\in\mathcal E}
\right),
\]
where \(\mathcal E\) denotes all successfully matched reallocation events.
A larger value indicates that the model more accurately ranks the relative
value of observed local movements.

\paragraph{Within-Item Concordance.}
To test whether a model adapts to the current allocation rather than assigning
a fixed channel preference to each item, we compare repeated events from the
same item. Let \(\mathcal P\) contain pairs \((e,e')\) from the same item whose
pre-decision contexts satisfy the prescribed similarity caliper. We compute
\[
C_{\mathrm{WI}}^{(b)}
=
\frac{1}{|\mathcal P|}
\sum_{(e,e')\in\mathcal P}
\left[
\mathbb I
\left(
\bigl(S_e^{(b)}-S_{e'}^{(b)}\bigr)
\bigl(R_e-R_{e'}\bigr)>0
\right)
+
\frac{1}{2}
\mathbb I
\left(
\bigl(S_e^{(b)}-S_{e'}^{(b)}\bigr)
\bigl(R_e-R_{e'}\bigr)=0
\right)
\right].
\]
The metric equals \(0.5\) under random pairwise ordering and approaches \(1\)
when score differences consistently agree with response differences.

\paragraph{Aligned--Reverse Response Gap.}
Let \(q_{\alpha}^{(b)}\) and \(q_{1-\alpha}^{(b)}\) denote the lower and upper
\(\alpha\)-quantiles of the score distribution, with
\(\alpha=0.2\). We define
\[
\mathcal E_b^{+}
=
\left\{
e:S_e^{(b)}\ge q_{1-\alpha}^{(b)}
\right\},
\qquad
\mathcal E_b^{-}
=
\left\{
e:S_e^{(b)}\le q_{\alpha}^{(b)}
\right\}.
\]
The response gap is
\[
G_{\mathrm{AR}}^{(b)}
=
\frac{1}{|\mathcal E_b^{+}|}
\sum_{e\in\mathcal E_b^{+}}R_e
-
\frac{1}{|\mathcal E_b^{-}|}
\sum_{e\in\mathcal E_b^{-}}R_e.
\]
This metric directly compares the empirical responses of movements most aligned
with the model against those ranked least favorably by the model.

\end{document}